\documentclass{article}

\usepackage[utf8]{inputenc} 
\usepackage[T1]{fontenc}    
\usepackage[hidelinks]{hyperref}       
\usepackage{url}            
\usepackage{booktabs}       
\usepackage{amsfonts}       
\usepackage{nicefrac}       
\usepackage{microtype}      
\usepackage{xcolor}         
\usepackage{array}          %
\usepackage{svg}

\usepackage[noend]{algorithmic}
\usepackage{amsmath,amsthm,amssymb, bbm}
\usepackage{authblk}
\usepackage{blindtext}
\usepackage{bm}
\usepackage{booktabs} 
\usepackage{comment}
\usepackage{enumitem}
\usepackage{fullpage}
\usepackage{graphicx}
\usepackage{hyperref}
\usepackage{multirow}
\usepackage{multicol}
\usepackage[square, comma, numbers,sort&compress]{natbib}
\usepackage{subcaption}
\usepackage{tikz} 
\usepackage{url}
\usepackage{xcolor}

\usepackage[ruled]{algorithm2e}

\SetAlFnt{\small}
\SetAlCapFnt{\small}
\SetAlCapNameFnt{\small}
\SetAlCapHSkip{0pt}
\IncMargin{-\parindent}

\usepackage{xcolor}

\usepackage{comment}
\usepackage{xspace}
\usepackage{sansmath}
\usepackage{amsmath,amsthm} 
\usepackage{mathtools}
\usepackage{amssymb} 
\usepackage{enumitem}
\usepackage[noabbrev,nameinlink]{cleveref}
\usepackage{placeins}
\usepackage{multicol}

\newcommand{\RR}{\mathbb{R}} 
\newcommand{\EE}{\mathbb{E}} 
\newcommand{\abs}[1]{\left| #1 \right|}

\newcommand{\calG}{\ensuremath{\mathcal{G}}\xspace}
\newcommand{\expect}[1]{\ensuremath{\mathbb{E}\left[#1\right]\xspace}}

\newcommand{\xmax}{\ensuremath{x_{\textnormal{max}}}\xspace}

\newcommand{\paren}[1]{\ensuremath{\left(#1\right)}\xspace}
\newcommand{\card}[1]{\left\vert{#1}\right\vert}
\newcommand{\eps}{\ensuremath{\varepsilon}}
\newcommand{\Ctilde}{\ensuremath{\widetilde{C}}}
\newcommand{\bracket}[1]{\left[#1\right]}

\newcommand{\polylog}{\ensuremath{\,\textnormal{polylog}}\xspace}

\renewcommand{\uplus}{U^+}
\newcommand{\uminus}{U^-}
\newcommand{\cplus}{C^+}
\newcommand{\cminus}{C^-}
\newcommand{\opt}{\text{OPT}}
\newcommand{\fair}{\text{Fair-OPT}}
\newcommand{\thresh}{x^*}
\newcommand{\pof}{\text{PoF}}
\newcommand{\pfail}{\ensuremath{p_{\textnormal{fail}}}\xspace}
\newcommand{\calX}{\ensuremath{\mathcal{X}}\xspace}
\newcommand{\Dtilde}{\ensuremath{\widetilde{D}}\xspace}

\newcommand{\xmin}{\ensuremath{x_{\textnormal{min}}}\xspace}
\newcommand{\Otilde}{\ensuremath{\widetilde{O}}\xspace}

\usepackage{tcolorbox}

\newtheorem{proposition}{Proposition}

\newtheorem{lemma}{Lemma}
\newtheorem{observation}[lemma]{Observation}

\theoremstyle{definition}
\newtheorem{definition}{Definition}

\newtheorem{Assumption}{Assumption}
\newtheorem{remark}{Remark}

\crefname{property}{property}{Property}
\creflabelformat{property}{(#1)#2#3}

\crefname{equation}{eq}{Eq}
\creflabelformat{equation}{(#1)#2#3}

\crefname{thm}{theorem}{Theorem}
\creflabelformat{Theorem}{(#1)#2#3}

\crefname{algocf}{Algorithm}{Algorithm}
\creflabelformat{algocf}{#1#2#3}

\newenvironment{tbox}{\begin{tcolorbox}[
		enlarge top by=5pt,
		enlarge bottom by=5pt,
		boxsep=0pt,
		left=4pt,
		right=4pt,
		top=10pt,
		arc=0pt,
		boxrule=1pt,toprule=1pt,
		colback=white
		]
	}
	{\end{tcolorbox}}

\usepackage{thmtools}
\usepackage{thm-restate}

\date{}
\title{Price of Fairness in Short-Term and Long-Term\\ Algorithmic Selections}
\author{Shahin Jabbari\thanks{Drexel University, shahin@drexel.edu}, Chen Wang\thanks{Rensselaer Polytechnic Institute, wangc33@rpi.edu} }

\begin{document}
\maketitle
\begin{abstract}
Algorithmic decision-making in high-stakes settings can have profound impacts on individuals and populations. While much prior work studies fairness in static settings, recent results show that enforcing static fairness constraints may exacerbate long-run disparities. Motivated by this tension, we study a stylized sequential selection problem in which a decision-maker repeatedly selects individuals, affecting both immediate utility and the population distribution over time. We introduce notions of group fairness for both the short and long term and theoretically analyze the trade-off between fairness and utility via the \emph{Price of Fairness} (PoF). We characterize optimal and fair policies in the short term and show that the PoF can be large even when group distributions are nearly identical. In contrast, we show that long-term disparities can vanish under simple investment policies that achieve a low PoF. We also empirically validate these theoretical observations using both synthetic and real datasets.
\end{abstract}

\section{Introduction}
\label{sec:intro}
Motivated by the societal impact of algorithmic decision-making, prior work has documented discriminatory behaviors and proposed mitigatory and fairness-imposing algorithms across classification~\cite{BerkHJKR18,BarocasHN19}, resource allocation~\cite{ElzaynJJKNRS19,DonahueK20,TsangWRTZ19}, and sequential decision-making~\cite{ReuelM24}. However, this literature largely focuses on static or short-horizon settings, emphasizing immediate impacts while abstracting away long-term consequences.

In their seminal work,~\citet{LiuDRSH18} showed that repeatedly enforcing static group-fairness notions from classification, such as demographic parity~\cite{CaldersV10} or equality of opportunity~\cite{HardtPS16}, can lead to stagnation and exacerbate group disparities when decisions influence the environment, sometimes even performing worse than simple utility maximization. Since then, various long-term fairness notions and algorithms—including those based on reinforcement learning—have been proposed to mitigate these effects~\cite{WenBT21,ChiSD+22}. However, existing results either establish strong impossibility or incompatibility results for jointly achieving fairness and high utility~\cite{JabbariJKMR17,DoroudiTB17}, or lack theoretical guarantees for optimal fair policies and thus offer no formal characterization of the fairness–utility tradeoff~\cite{WenBT21,HuZ22}.

\smallskip \noindent\textbf{Our Results and Contributions.} 
Motivated by this gap, we study a stylized selection problem in which a decision-maker (e.g., a financial institution) sequentially selects a set of individuals (e.g., to lend money) in each round. In our model, individuals belong to different groups and are represented by single-dimensional features (e.g., credit score). These decisions and their consequences (e.g., default or full payback) affect the decision-maker's utility and the distribution of individuals for future rounds. Given the potential disparity in the distributions, the decision-maker's selection policy can favor specific groups or individuals, raising concerns about unfairness and \emph{disparate treatment} or \emph{impact}. Furthermore, these disparities can worsen over time as population distribution evolves. We propose a \emph{new notion of group fairness}, both for the short and the long terms, requiring the group difference between the means of score distributions to be bounded.

Our primary goal is to understand the theoretical foundation of the trade-off between fairness violations and utility maximization in both the short and long terms. To this end, we define the notion of \emph{Price of Fairness} (PoF) for this domain, which captures the cost of utility due to fairness constraints. We characterize optimal and fair selection policies in the short term, and show that threshold-based policies can achieve optimal utilities. We then show that in the short term, the PoF can be arbitrarily high even when the groups have very similar initial distributions. 

The characterization of optimal and fair policies becomes significantly more complex in the long term, where current decisions interfere with future states of the population. In the single-step setting, fairness constraints often force a strict trade-off between utility and equity. However, we show that in the multi-step setting, this trade-off is largely an artifact of the short horizon. We analyze a class of \emph{long-term investment} policies, i.e., simple investment strategies that prioritize score improvement and prove that they guarantee a low PoF. We also empirically validate these theoretical observations using both synthetic data and real-world FICO credit score datasets.

\section{Related Work}
\label{sec:related}

\smallskip \noindent\textbf{Fairness in Sequential Decision-Making.}
Fairness in dynamic and repeated decision-making has been widely studied; see~\cite{ReuelM24}. Early work considers stationary bandit settings~\cite{LiuRDMP17,JosephKMR16}. In general MDPs,~\citet{JabbariJKMR17} formalize meritocratic fairness and show its incompatibility with utility maximization, a phenomenon also noted in~\cite{DoroudiTB17}. Subsequent work on fair reinforcement learning imposes population-level constraints on average reward or return~\cite{WenBT21,ChiSD+22,HuZ22}, emphasizing allocative equality. Instead, we focus on restorative justice and substantive equality (see Remarks~\ref{remark:substantive-equality}-\ref{remark-qualification}). Other related approaches include state-visit constraints~\cite{GhalmeNPZ22}, minimax fairness~\cite{EatonHK+25}, and welfare-based objectives~\cite{CousinsAL+24}; see also~\cite{SatijaLP+23,DengSW+23,ZhangTL+20,YinRLL23}. Our notion of fairness is inspired by state-based constraints~\cite{GhalmeNPZ22}, but we focus on characterizing optimal fair policies and analyzing the fairness–utility tradeoff. Unlike intervention-based approaches in dynamical systems~\cite{RateikeVF24}, we quantify the Price of Fairness in a fixed environment and show it vanishes under a class of \emph{investment} policies. Our work is distinct from fair division~\cite{SalemIN23,LiuH25,SinhaJBMH23}.

\smallskip\noindent\textbf{Fairness in Stylized Dynamic Settings.}
Fairness has also been studied in stylized dynamic models. \citet{ArunachaleswaranKRZ21} analyzes a multi-layer life-stage model with costly interventions in transition dynamics (see also~\cite{AcharyaAK0Z23}). Intergenerational dynamics driven by wealth and talent are studied in~\cite{HeidariK21}. In a continuous-time model,~\citet{MouzannarOS19} characterizes when affirmative action constraints improve long-run outcomes despite short-term costs. Simulation-based studies appear in~\cite{DAmourSABSH20}, while causal distribution shifts are modeled in~\cite{HuZ22}. Finally,~\citet{WuAHS25} identifies conditions under which Rawlsian policies outperform utilitarian policies in the long run.

\smallskip \noindent\textbf{Static Fairness Notions in Repeated Settings.}
\citet{LiuDRSH18}~and~\citet{ZhangTL+20} study a simple dynamic model and show that repeated application of myopic fair policies, such as demographic parity~\cite{CaldersV10} and equality of opportunity~\cite{HardtPS16}, can cause stagnation and can further exacerbate the disparities between different populations over harm. In contrast, the optimal policy that ignores fairness at least does not cause stagnation, though it cannot always rectify the disparities. While our stylized formulation is inspired by prior work such as~\cite{LiuDRSH18}, we aim to study a new notion of short- and long-term fairness (similar to works in fair reinforcement learning) instead of showing notions from classification can cause harm in dynamic settings. 

\smallskip \noindent\textbf{Constrained MDPs and Safe Reinforcement Learning.} 
Since most fairness definitions can be written as a constraint, our framework is closely connected to the literature on constrained MDPs and POMPDPs~\cite{Altman99,WenT18,AchiamHTA17}. However, these approaches primarily focus on constraints restricting a state-dependent cost function. In contrast, fairness constraints specifically aim to equalize statistics across different groups in some manner. There has also been significant recent interest from the reinforcement learning community to enforce \emph{safety} constraints. Most model-free approaches require access to simulators to generate samples from arbitrary state-action pairs~\cite{DingZBJ20}. Model-based approaches are generally less efficient in space and time complexity~\cite{LiuZKKT21,DingWYWJ21}. More recently,~\citet{CalvoFullanaPCR24,MullerAC+24} studied a setting with multiple reward functions and optimized the cumulative reward subject to a lower bound on the value of each reward function. These approaches study POMDPs/MDPs in full generality. Our goal is to gain a deeper theoretical understanding of the properties of fair interventions and their long-term effects through our stylized model. Recent studies investigate the trade-off between simple policies and optimality from a theoretical angle, but for simple MDPs~\cite{MansourMR22,SilvaCG22,ChengYX25}. Our work contributes to this emerging body of work.
\vspace{-2mm}
\section{Single-step Setting}
\label{sec:single-step}
\subsection{Model and Notation}
\label{sec:framework}
We start by presenting our stylized model of sequential decision-making, which mainly follows the model proposed in the seminal work of~\cite{LiuDRSH18}. We assume the information about each individual is aggregated using a one-dimensional score $x$ in a bounded range $X$ (e.g., $X$ can be the credit scores in lending). The score can be interpreted as an indication of an individual's perceived ability, with higher scores being more desirable.
We assume individuals belong to two groups: $g\in\{A, B\}$. We use $D_A$ and $D_B$ to denote the distributions over scores for each group, capturing the variation in scores within each group. We use $\mu_A$ and $\mu_B$ to denote the mean score for each group. Without loss of generality, we assume $\mu_A\geq \mu_B$ and, hence, group $A$ to have an advantage over group $B$. We use $w_g$ to denote the fraction of the population that belongs to group $g$, so $w_A+w_B=1$. 

We assume the decision-maker deploys a group-dependent policy $\pi: X\times g\to[0,1]$ to select individuals with a given score from each group, interpreted as the selection probability of an individual with a score $x$ within each group. We use $\pi_A$ and $\pi_B$ to denote the policies corresponding to each group, hence, $\pi=\{\pi_A, \pi_B\}$. When selected, each individual with a score $x$ has a group-agnostic probability of success denoted by $p: X\to [0,1]$. Similar to \cite{LiuDRSH18}, we make the following assumption on $p$, demonstrating that success probabilities are positively correlated with the scores. 
\begin{Assumption}[\cite{LiuDRSH18}]
\label{assumpt:p}
$p: X\to [0,1]$ is a monotonically increasing function of the score $x$.
\end{Assumption}

When an individual with score $x$ is selected, the decision-maker receives a fixed profit of $\uplus \geq 0$ if the individual succeeds (which happens with probability $p(x)$). Otherwise, the decision-maker incurs a cost of $\uminus < 0$ (which happens with probability $1-p(x)$). Thus, for a selected individual with score $x$, the expected \emph{utility} $u:X\to \RR$ of the decision-maker can be written as $\EE\left[u(x)\right] = p(x) \uplus + \left(1 - p(x)\right) \uminus$. The success of the individual can also impact their score. In particular, we assume, upon success, the score increases by $\cplus\geq 0$ (which happens with probability $p(x)$). Otherwise, the score decreases by $\cminus<0$ (which happens with probability $1-p(x)$).\footnote{Since the scores are bounded, the updated scores should be projected/clipped to guarantee that they fall inside of $X$.} Thus, the expected \emph{change} $\Delta: X\to\RR$ in score for an individual with score $x$ can be written as
$\EE[\Delta(x)] = p(x) \cplus + (1 - p(x)) \cminus.$ For individuals who are not selected, we assume the score to remain the same and the decision-maker to receive 0 utility.

By Assumption~\ref{assumpt:p}, both the expected utility and the expected score change are monotone functions of the score. Similar to~\cite{LiuDRSH18}, we assume the utility function for individuals is stricter than the score change function. 
\begin{Assumption}[\cite{LiuDRSH18}]
\label{assumpt:struc_assumption}
$\uplus/\uminus \geq \cplus/\cminus$.
\end{Assumption}
Assumption~\ref{assumpt:struc_assumption} implies that if $\EE[u(x)]\geq 0$ then $\EE[\Delta(x)]\geq 0$, which is an important property for the non-existence of extractive (exploitative) outcomes, where the decision-maker can profit from the \emph{long-term decrement} of the individual scores.
We study the relaxation of this assumption in Section~\ref{sec:exp}. See Definition~\ref{def:score-category} and Observation~\ref{obs:no-exploitative} for more details.

\begin{table*}[h]
\centering
\begin{tabular}{|l |l |}
\toprule
\textbf{Assumption} & \textbf{Justification} \\
\midrule
Assumption~\ref{assumpt:p} & Same as \cite{LiuDRSH18}, higher score individuals have higher chances to succeed. \\
\midrule
Assumption~\ref{assumpt:struc_assumption} & Same as \cite{LiuDRSH18}, the utility function for individuals is stricter than the score change function. \\
\bottomrule
\end{tabular}
\caption{Justification of assumptions in the single-step setting.\label{tab:assumptions-single}}
\end{table*}

An instance $I$ of the selection problem is determined by the description of $X, D_A, D_B, w_A, w_B, p, \uplus, \uminus, \cplus,$ and $\cminus$. The expected utility for any selection policy $\pi$ is defined as
\begin{align}
\label{eq:utility}
V(\pi, I) &= 
\sum_{g \in \{A, B\}} w_g \sum_{x \in X} \pi_g(x) D_g(x) \EE[u(x)].
\end{align}
The decision-maker's goal in the single-step setting is to compute an \emph{optimal} policy $\pi$ maximizing the total expected utility:
\begin{align}
\label{eq:opt}
\opt(I)&:=\max_{\pi=\{\pi_A, \pi_B\}, V(\pi, I) \geq 0} V(\pi, I),
\end{align}
where $\opt(I)$ denotes the utility of the optimal policy.

\subsection{Fairness Definition}

The groups start with different distributions of scores, and the decision-maker's policy can further exacerbate this difference, raising concerns about \emph{disparate impact} and unfairness of the selection policy. It is well known that static intervention policies such as equal selection rates (statistical parity) or equal selection rates conditioned on qualification status (equality of opportunity) can further worsen the discrepancies between the groups~\cite{LiuDRSH18,ZhangTL+20}. To alleviate this, inspired by~\cite{LiuDRSH18}, we define a fairness constraint on the decision-maker's policy requiring the difference between the means of the distributions for different groups after deploying the policy to be bounded by a parameter $\alpha \geq 0$. Formally, let $\delta = \mu_A - \mu_B \geq 0$ denote the mean score difference between the two groups,
capturing the initial disparity.  Also let 
\begin{align}
\mu'_g(\pi_g) 
&= \mu_g+\sum_{x \in g} \pi_g(x) D_g(x) \EE\left[\Delta(x)\right],
\end{align}
for $g\in \{A, B\}$ to denote the mean after deploying $\pi$.
Therefore, $\delta'(\pi)=\abs{\mu'_A(\pi_A) - \mu'_B(\pi_B)}$ represents the post-deployment difference in the means. We next propose our fairness definition in the single-step setting. 
\begin{definition}
\label{def:fairness}
A policy $\pi=\{\pi_A, \pi_B\}$ satisfies $\alpha$-fairness for some $\alpha \geq 0$ if $\delta'(\pi) \leq \alpha.$
\end{definition}

The parameter $\alpha$ controls the degree of mean disparity \emph{after} policy deployment, and smaller $\alpha$ values correspond to more restrictive policies. When it is clear from the context, we drop the dependency of $\mu'_g(\pi_g)$ on $\pi_g$ and simply write $\mu'_g$.

\begin{remark}
\label{remark:substantive-equality}
Our notion of fairness focuses on the states (scores) as opposed to outcomes (utilities)~\cite{ChiSD+22}, preventing negative feedback loops that~\cite{LiuDRSH18} have warned about. It also enforces restorative justice~\cite{Davis21} and substantive equality of opportunity~\cite{HeidariLGK19}, requiring the decision-maker to \emph{invest} in the group with a lower average score. See Section~\ref{sec:repeated}.
\end{remark}

\begin{remark}
\label{remark-qualification}
Given Assumption~\ref{assumpt:p} on the direct correlation of scores and qualification status, the closeness of average scores, indirectly, enforces the closeness of average qualification status between the groups. Similar to~\cite{LiuDRSH18}, we define fairness using the scores, i.e., the scores are changing, and the score-dependent qualification scores remain fixed. See~\cite{ZhangTL+20} that takes the alternative approach, where scores are fixed, and the qualification statuses are changing. 
\end{remark}

An \emph{optimal fair} policy, for instance $I$ in the single-step setting, is a policy $\pi$ that maximizes the decision-maker’s expected utility subject to satisfying the fairness constraint in Definition~\ref{def:fairness}. Formally, this can be written as follows:
\begin{equation}
\label{eq:opt-fair}
\fair(I) := \max_{\pi=\{\pi_A, \pi_B\}, V(\pi, I) \geq 0} V(\pi, I),
\text{s.t.  } \pi \text{ is } \alpha-\text{fair}.
\end{equation}
We use $\fair(I)$ to denote the utility of the optimal fair policy. In the multi-step setting, the decision-maker selects individuals repeatedly over rounds, causing the scores to change in each step. We discuss that setting in Section~\ref{sec:repeated}.
\subsection{Characterization of the Optimal Policies}
\label{sec:opt-fair-character}
To study the short-term effects of selection policy, in this section, we start by characterizing the optimal policy (Proposition~\ref{pro:opt}) and optimal fair policies (\Cref{thm:thresh_opt}) in the single-step setting. We first define the \emph{categories} of the individual scores based on the expected utility and score change. 

\begin{definition}
    \label{def:score-category}
        For any score $x\in X$, we say $x$ is in one of the following categories:
    \begin{itemize}
        \item \textbf{Profitable and Improving}: We say $x$ is in category $1$, denoted as $x\in C_1$, if $\expect{\Delta(x)}\geq 0$ and $\expect{u(x)}\geq 0$.
        \item \textbf{Extractive Candidates}: We say $x$ is in category $2$, denoted as $x\in C_2$, if $\expect{\Delta(x)}< 0$ and $\expect{u(x)}\geq 0$.
        \item \textbf{Investment Candidates}: We say $x$ is in category $3$, denoted as $x\in C_3$, if $\expect{\Delta(x)}\geq 0$ and $\expect{u(x)}<0$.
        \item \textbf{Unprofitable and Degrading}: We say $x$ is in category $4$, denoted as $x\in C_4$, if $\expect{\Delta(x)}< 0$ and $\expect{u(x)}< 0$.
    \end{itemize}
\end{definition}

Our first crucial observation is that under Assumption~\ref{assumpt:struc_assumption}, the exploitative outcomes do not exist, i.e., we do \emph{not} have category $2$ with $\expect{\Delta(x)}< 0$ and $\expect{u(x)}\geq 0$. See Appendix~\ref{sec:appendix-opt-fair-character}.
\begin{observation}
\label{obs:no-exploitative}
Under Assumption 2, the set of Extractive Candidates ($C_2$) is empty, i.e., $C_2 = \emptyset$.
\end{observation}

We want to argue that the optimal (utility-maximizing) policy has to be a threshold policy.
To this end, we start by formally defining threshold policies.
\begin{definition}[Threshold policy]
\label{def:threshold-policy}
A policy $\pi$ is a \emph{threshold policy} if there exists $\thresh\in X$ and $\omega\in[0,1]$ such that
\[
\pi(x) =
\begin{cases}
1, & \text{if } x > \thresh, \\
\omega, & \text{if } x = \thresh, \\
0, & \text{if } x < \thresh.
\end{cases}
\]
\end{definition}
In Definition~\ref{def:threshold-policy}, $\thresh$ represents the threshold score, and $\omega$ is the fraction of selected individuals at the threshold score. Threshold policies are desirable since they offer simplicity and interpretability~\cite{Corbett-DaviesG23}. Proposition~\ref{pro:opt} states that the optimal policy selects only individuals with positive expected utility, exactly characterized by category 1.
\begin{proposition}[\cite{LiuDRSH18}]
\label{pro:opt}
Under Assumption~\ref{assumpt:p}, the optimal policy in the single step is a group-agnostic threshold policy that selects all the individuals in category 1, i.e., individuals with score $x$ such that $p(x)\geq -\uminus/(\uplus-\uminus)$.
\end{proposition}

We next show that optimal fair policies, when they exist, are also thresholding policies. However, unlike optimal policies, the threshold can be different for each group.
\begin{restatable}{theorem}{TheoremThresholdOPT}
\label{thm:thresh_opt}
Suppose the optimization problem in Equation~\ref{eq:opt-fair} has a feasible solution. Under Assumptions~\ref{assumpt:p} and \ref{assumpt:struc_assumption}, in the single-step setting, there exists an optimal fair policy that is a threshold policy for each group.
\end{restatable}
This observation is also consistent with prior work~\cite{Corbett-DaviesG23}. See Appendix~\ref{sec:appendix-opt-fair-character} for proof. 
\subsection{Price of Fairness}
\label{sec:pof}
We next study the relationship between the expected utility of optimal and optimal fair policies in the single-step setting. We measure the cost of enforcing the fairness constraint as the fraction of expected utility retained by the optimal fair policy compared to the optimal policy in the worst case over all instances. In line with prior work studying fairness in different decision-making settings~\cite{Bertsimas11,MenonW18,ElzaynJJKNRS19,DonahueK20}, we call this notion \emph{price of fairness}.
\begin{definition}
\label{def:pof}
The \textit{Price of Fairness} (PoF) is defined as:
\begin{equation}
\label{eq:pof}
\pof = \max_{I}\left( 1 - \frac{\fair(I)}{\opt(I)}\right).
\end{equation}
\end{definition}
The PoF is in $[0,1]$, and higher PoF values mean that enforcing fairness requires a higher degree of utility degradation. Hence, lower values of the price of fairness are more desirable. In~\Cref{pro:lb-pof}, an $\alpha$-dependent lower bound on PoF indicates that in the worst-case, the PoF can get arbitrarily close to 1 in the single-step setting when $\alpha$ goes to 0.
\begin{restatable}{theorem}{ThmPoFGeneral}
\label{pro:lb-pof}
Suppose $X=[0,1]$. Then, in the single-step setting, $\pof \geq 1 - O(\alpha)$. 
\end{restatable}
The proof relies on creating a family of instances $I$ and a characterization of the optimal and optimal fair policies for this family. We defer the proof to Appendix~\ref{sec:appendix-opt-fair-character}. To show a large PoF, the proof relies on constructing an instance in which the population distributions are very different (i.e., the total variation distance of the two distributions is 1). This raises a natural question: can high PoF be avoided if the distributions are more similar (e.g., have a smaller total variation distance)? Before answering this question, we require the following definition.

\begin{definition}
\label{def:degrading}
A policy $\pi$ is \emph{non-degrading} if $\pi(x) = 0,\forall x\in C_4$, i.e., no allocation to individuals in $C_4$ category.
\end{definition}
Recall that an individual $x$ belongs to category 4 if $\EE[u(x)] < 0$, $\EE[\Delta(x)] < 0$. A population-degrading policy selects individuals that not only impose negative expected utility but also degrade population scores. It is known that equality-enforcing fairness notions (such as our notion) can lead to population degradation~\cite{RahmattalabiJLV21,StoicaHC20}. 
Our next result shows that it is not possible to avoid high PoFs for all non-degrading policies.
\begin{restatable}{theorem}{ThmPoFNondegrade}
\label{thm:pof-tv}
For any $\epsilon > 0$, in the single-step setting, there exists an instance such that the total variation distance between $D_A$ and $D_B$ is at most $\epsilon$, but the PoF is $1 - O(\alpha)$. 
\end{restatable}

See Appendix~\ref{sec:appendix-opt-fair-character} for proof. 
Finally, in Proposition~\ref {pro:pof-sufficient-condition} (Appendix~\ref{sec:appendix-opt-fair-character}), we provide sufficient conditions where more favorable PoFs are achievable.
These results indicate that for the single-step setting, in general, the PoF can be quite large.
\newcommand{\gsuccess}{\ensuremath{g_{\textnormal{success}}}\xspace}

\section{Multi-step Setting}
\label{sec:repeated}
We now focus on the long-term effects of the selection policy by studying the multi-step setting. 
In this setting, the decision-maker makes selections sequentially for $T$ rounds. The decisions in the previous steps $1, \ldots, t-1$ will affect the distribution at time step $t$. 
The decision-maker likes to maximize the total utility over all time steps. 
We show that, unlike in the single-step setting, where the PoF could be high, we can obtain a low PoF using simple investment policies.

\subsection{Model and Notation}
\label{subsec:multi-step-model-and-notation}
For simplicity, we assume the scores are integers from 0 to $\xmax$. We work extensively with the \emph{categories} of the scores as defined in Definition~\ref{def:score-category}. 
For convenience, we use $g\in \{A, B\}$ and $c\in \{1,3,4\}$ to refer to individuals in each group and category, respectively. 
For example, we use $A\cap C_1$ to denote individuals from group $A$ that belong to category $C_1$. 
For each category $c$ and group $g$, we define $\Gamma^{c}_g$ and $\mu^{c}_{g}$ as the total and average scores of the specific category $c$ and group $g$, i.e., 
\begin{align*} 
\Gamma^{c}_{g}= \sum_{x \in g \cap c} x; \qquad \mu^{c}_{g}=\frac{1}{\card{g \cap c}}\cdot \Gamma^{c}_{g}. 
\end{align*}

Let $T$ be the total number of time steps. We use $D^{(t)}_g$ to denote the distribution of the scores for group $g\in \{A,B\}$ at step $t$. The multi-step allocation model is a stochastic process, where the decision maker selects a sequence of policies $\Pi_{T} =\{\pi^{(t)}\}_{t=1}^{T}$. 
We also track the change of scores during the stochastic process: let $x$ be any score at time step $1$, the notation $G(x)^{t}$ is defined as the \emph{set of scores} after evolving for $t$ steps\footnote{More formally, $G(x)^{t}$ is a realization of the random variable denoting the set of score. See Appendix~\ref{subsec:detailed-multi-step-objective} for a rigorous definition.}.
The multi-step utility can be written as follows.

\begin{equation}
\label{equ:multi-step-utility-simplified}
\begin{split}
& \max_{\Pi_{T}} \expect{\sum_{t=1}^{T} \sum_{g \in {A,B}} w_g \sum_{x \in X} D^{(t)}_g(x) \pi^{(t)}_g(x) u(x)}\\
& =\max_{\Pi_{T}}\sum_{t=1}^{T} \sum_{g \in {A,B}} w_g \sum_{x \in X} D^{(t)}_g(x) \pi^{(t)}_g(x) \expect{u(x)},
\end{split}
\end{equation}
where the equality holds due to the linearity of expectation. However, note that the random variables in different steps are \emph{not} independent since the past realization will affect $\pi^{t}$ and $D^{t}$ (see Appendix~\ref{subsec:detailed-multi-step-objective} for more insight about the dependency).

To define a multi-step analog of Definition~\ref{def:fairness} for fairness, we could similarly define $\mu'_{g}(\Pi_{t})$ as the expected average score for group $g\in \{A,B\}$ for the $t$-th time step, and $\delta'(\Pi_{t},t)=|\mu'_{A}(\Pi_{t})-\mu'_{B}(\Pi_{t})|$. 

We restrict our attention to rational (non-degrading) policies and introduce the multi-step version of the non-degrading policy, which generalizes~Definition~\ref{def:degrading}. 
\begin{Assumption}
\label{assump:no-C4-in-optimal-policy}
A policy $\pi$ is \emph{rational} (\emph{non-degrading}) in the multi-step setting, if for all $t\in[T]$, $g\in\{A,B\}$ and $x\in C_4$, $\pi^{(t)}_g(x)=0$. 
\end{Assumption}

Assumption~\ref{assump:no-C4-in-optimal-policy} says that no policy will allocate resources to an individual expected to be unprofitable in \emph{both the short and long terms}, implying that the decision-maker is \emph{rational}. 
Our next assumption characterizes the advantage group $A$ over $B$. 
\begin{Assumption}
\label{assump:real-advantage}
    We assume that the scores for group $A$ in categories $C_1$ and $C_3$ are greater than the scores from the same categories for group $B$ by a positive value $\beta > 0$, i.e., 
    \begin{align*}
    \frac{\Gamma_{A}^{C_1}+ \Gamma_{A}^{C_3}}{\card{A}} - \frac{\Gamma_{B}^{C_1}+ \Gamma_{B}^{C_3}}{\card{B}} > \beta.
    \end{align*}  
\end{Assumption}

\begin{table*}[h]
\centering
\begin{tabular}{|l | p{2.5cm} | p{12cm}|}
\toprule
\textbf{Assumption} & \textbf{Setting} & \textbf{Justification} \\
\midrule
Assumption~\ref{assump:no-C4-in-optimal-policy} & Single \& Multi-opportunity & The decision-maker is rational and avoids individuals expected to yield net utility losses in both the short and long term. \\
\midrule
Assumption~\ref{assump:real-advantage} & Single \& Multi-opportunity & Ensures Group A's advantage stems from potentially profitable candidates rather than simply having fewer unprofitable ones. \\
\midrule
Assumption~\ref{assump:low-failure-prob-C1-C3} & Single-opportunity & Models ``stability'' where profitable candidates are trustworthy, effectively controlling the variance of the score trajectory. \\
\midrule
Assumption~\ref{assump:exponential-decrease-of-fail-prob} & Single-opportunity & Aligns with natural distributions (e.g., logistic) where success probability improves rapidly in low-to-mid score ranges. \\
\midrule
Assumption~\ref{assump:change-of-score} & Multi-opportunity & Reflects real-world scoring systems (e.g., FICO) where scores are treated as discrete integers and updates occur in integer increments. \\
\bottomrule
\end{tabular}
\caption{Justification of assumptions in the multi-step setting.\label{tab:assumptions}}
\end{table*}

Assumption~\ref{assump:real-advantage} states that the advantage of group $A$ is \emph{not} due to a strange scenario where the group has no ``real advantage'' among the individuals who could get allocations, and it simply has a smaller number of individuals with very bad scores.

We consider two settings for multi-step allocation: the single-opportunity setting, where the decision-maker adjusts scores based on a single outcome, and the multi-opportunity setting, where the decision-maker modifies scores only after the outcomes of multiple independent decisions. Table~\ref{tab:assumptions} provides a summary of the assumptions used in this section.

\smallskip \noindent\textbf{Technical Overview.} 
At a high level, the reason we can obtain low PoF is due to the \emph{positive drift} for individuals in categories $C_1$ and $C_3$ and the existence of the $\xmax$ score.
If we only look at the \emph{expectation}, the intuition for low PoF is straightforward and clean: for any score $x\in C_1 \cup C_3$, we have $\expect{\Delta(x)}>0$ then, when $T\rightarrow \infty$, the scores for both advantaged and disadvantaged groups are \emph{expected} to hit $\xmax$.
Therefore, unless the gap between the groups is induced by $C_4$ (which is a very strange situation since the advantaged group has no ``real advantage''), the gap between the mean scores should decrease. 
Furthermore, for the utility, the decision-maker only pays the ``burn-in'' part for individuals in $C_1 \cup C_3$ to reach $\xmax$; and afterwards, the decision-maker is expected to get the maximum possible utility. This summarizes the intuitive reason for the PoF to be low.

The actual analysis, however, is much more challenging and complicated. The challenge lies in the \emph{expectation vs. realization}. 
Since multi-step allocation is a stochastic process, the \emph{realization of the randomness} in the previous steps interferes with the current and future steps.
An adversarial case can be as follows: in one step, many scores in $C_3$ drop to $C_4$, and the high expected score and utility can never materialize.

We overcome the challenge in two ways: $i)$ for the single-opportunity setting, we introduce additional realistic assumptions to bound the \emph{variance} of the scores; and $ii)$ in the multi-opportunity setting, we directly leverage the multiple trials to introduce concentration.
The analysis in the single-opportunity setting is perhaps more intellectually stimulating. 
Here, we only assume the failure probability for individuals in $C_1 \cup C_3$ is bounded, and the success probability increases fast in the low-to-mid score ranges--both are considered common in practice.
We explicitly track the change of scores by bounding the probability of ``failure to increase'', and we prove that the set of scores that cannot increase forms a \emph{geometrically decreasing series} with high probability. 
Hence, we can ``peel off'' this small fraction of ``bad scores'', and use the scores that kept increasing to argue both improved fairness and low PoFs.

\subsection{Price of Fairness}
\smallskip \noindent\textbf{Single-opportunity Setting.}
\label{subsec:multi-step-single-opportunity}
In the \emph{single-opportunity} setting, the decision-maker adjusts scores \emph{every time} after an individual succeeds or fails. 
In this setting, with specific but natural assumptions, we can achieve fairness low PoF as $T$ increases.
We first introduce the following assumptions used \emph{exclusively for the single-opportunity setting}. 
\begin{Assumption}
\label{assump:low-failure-prob-C1-C3}
     For any $x\in C_1 \cup C_3$, we assume the probability of failure is at most $\frac{\beta}{N\cdot \xmax}$ for sufficiently large constant $N$, i.e., $1-p(x)\leq \frac{\beta}{N\cdot \xmax}$.
     Here, $\beta$ is the parameter in Assumption~\ref{assump:real-advantage}, and $\xmax$ is the maximum score.
\end{Assumption}
Furthermore, in the single-opportunity setting, we control the variance of the realization of the score by adding another assumption on the \emph{failure probability}.
\begin{Assumption}
\label{assump:exponential-decrease-of-fail-prob}
We assume for any score $x\in \calX$, $1-p(x+C^+) \leq \frac{1}{3}\cdot \paren{1-p(x)}$. 
Furthermore, we assume the individuals with the maximum score always succeed, i.e., $p(\xmax)=1$. 
\end{Assumption}
In practice, Assumption~\ref{assump:exponential-decrease-of-fail-prob} holds when the probability of success increases drastically during the low-to-mid score ranges.
One natural family of distributions in this regime is \emph{logistic distributions}, and we provide a justification in Appendix~\ref{subsec:assump-justification}.
We show a structural lemma that, if we keep selecting individuals in $C_1 \cup C_3$, we will have a good probability to receive a ``cascading'' effect for a majority of such individuals. This is stated in Lemma~\ref{lem:multi-step-variance-control} (proof in Appendix~\ref{sec:appendix-repeated}). Recall that $G(x)^{t}$ is the set of scores starting from $x$ and evolving till $t$-th step. 

\begin{lemma}
    \label{lem:multi-step-variance-control}
    Suppose for any $x\in C_1 \cup C_3$, the success probability $p(x)$ satisfies $p(x)\geq 1-\pfail$.
    Then, under Assumption~\ref{assump:exponential-decrease-of-fail-prob}, with probability at least $99/100$, for all but at most $O(\pfail)$ fraction of scores in $G(x)^{(t)}$, we have
    \begin{align*}
        \card{x^{(t)}-\expect{x^{(t)}}} \leq \sqrt{t}\cdot (C^+-C^-) \polylog(t \cdot \card{C_1 \cup C_3}).
    \end{align*}
    Therefore, when $t\rightarrow \infty$, we have 
    \[\card{\{x\in G(x)^{t}\mid x=\xmax\}}\geq (1-O(\pfail))\cdot \card{G(x)^{t}}.\]
\end{lemma}

We now introduce the main results for the multi-step single-opportunity setting. We argue that, unlike the single-step setting, for the multi-step setting, if the number of steps is sufficiently large, we could use the following simple long-term investment policy to simultaneously satisfy fairness and obtain low PoF. This is stated in~\Cref{prop:muilti-step-single-opportunity} (proof in Appendix~\ref{sec:appendix-repeated}).

\vspace{-6pt}
\begin{tbox}
\textbf{A simple long-term investment policy.}
\begin{itemize}
\item Select individuals in both $A$ and $B$ groups who have always succeeded in the previous round (with Assumption~\ref{assump:no-C4-in-optimal-policy}, this is a subset of individuals in $C_1$ and $C_3$).
\end{itemize}
\end{tbox}
\vspace{-8pt}

\begin{restatable}{theorem}{ThmMultiStepSingleOpportunity}
\label{prop:muilti-step-single-opportunity}
Assume Assumptions \ref{assump:no-C4-in-optimal-policy}, \ref{assump:real-advantage}, \ref{assump:low-failure-prob-C1-C3}, and \ref{assump:exponential-decrease-of-fail-prob} hold.
Then, with probability at least $99/100$, for sufficiently large $T$, the long-term investment policy guarantees $\delta'(T)< |\mu_A-\mu_B|-\beta/2$, where $|\mu_A-\mu_B|$ is the \emph{initial gap} of the mean scores of the groups, and $\beta$ is the parameter in Assumptions~\ref{assump:real-advantage} and \ref{assump:low-failure-prob-C1-C3}.

Furthermore, we have that $\lim_{T\rightarrow \infty} \pof =O(\pfail)$ as long as $\alpha\geq \max\{\frac{\Gamma^{C_4}_{A}}{\card{A}}, \frac{\Gamma^{C_4}_{B}}{\card{B}}\}+\beta$
\end{restatable}

The conditions in~\Cref{prop:muilti-step-single-opportunity} implies that the statement holds for $\alpha\geq \max\{\mu^{C_4}_{A}, \mu^{C_4}_{B}\}+\beta$ as well, although the statement is much weaker.

\smallskip \noindent\textbf{Multi-opportunity Setting.}
\label{subsec:multi-opportunity}
We further investigate a different setting, where the decision-maker will keep selecting for \emph{multiple times} before committing to a decision that would be used for the rest of the time steps. 
We will show that in this setting, the PoF goes to $0$ with high probability with assumptions much weaker than Assumptions \ref{assump:low-failure-prob-C1-C3} and \ref{assump:exponential-decrease-of-fail-prob}.  
First, we formally define the \emph{multi-opportunity} setting in the multi-step setting.

\begin{definition}
\label{def:multi-opportunity-setting}
Let $x^{(t)}$ be any score at time step $t$. In the multi-step setting with \emph{number of opportunities} $m$, the decision-maker selects an individual $m$ times, and makes selection decisions after observing \emph{all} outcomes of the $m$ selections. 

Let $\gsuccess$ be the fraction of the times the individual succeeds. 
\[\Delta(x^{(t)}) = \gsuccess\cdot C^+ + (1-\gsuccess)\cdot C^-.\]
\end{definition}

Note that the \emph{expected} score change remains the same:
\begin{align*}
\expect{\Delta(x^{(t)})} & = \expect{\gsuccess\cdot C^+ + (1-\gsuccess)\cdot C^-}\\
&= \expect{\gsuccess} \cdot C^+ + (1-\expect{\gsuccess})\cdot C^- \\
&= p(x)\cdot C^+ + (1-p(x))\cdot C^-.
\end{align*}
We also introduce a much milder assumption about the \emph{expectation} of the change of score.
\begin{Assumption}
	\label{assump:change-of-score}
	For any score $x\in \calX$, the \emph{expectation} of the change of scores, i.e., $\expect{\Delta(x)}$, is a \emph{positive integer}. 
\end{Assumption}

In the multi-opportunity setting, we could simply select all individuals in $C_1$ and $C_3$ to get a low PoF.
\vspace{-7pt}
\begin{tbox}
	\textbf{An even simpler long-term investment policy.}
	\begin{itemize}
		\item Select individuals from categories $C_1$ and $C_3$ in both groups $A$ and $B$. 
	\end{itemize}
\end{tbox}
\vspace{-7pt}

We formally show the guarantees for the ``even simpler investment policy'' as follows. Note that the quantifiers here are different from those of \Cref{prop:muilti-step-single-opportunity}: we first fix the time horizon $T$ before giving the parameters for the multi-opportunity setting. We defer the proof to Appendix~\ref{sec:appendix-repeated}.

\begin{restatable}{theorem}{ThmMultiStepMultiOpportunity}
\label{prop:muilti-step-multi-opportunity}
Assume Assumptions \ref{assump:no-C4-in-optimal-policy}, \ref{assump:real-advantage}, and \ref{assump:change-of-score} hold.
For any sufficiently large $T$ and $\delta>0$, suppose we have $m\geq 5(C^+-C^-)\cdot \log{T}$ in the multi-opportunity setting.
Then, with probability at least $99/100$, the even simpler long-term investment policy guarantees $\delta'(T)< \delta$ and $\pof =0$ as long as $\alpha\geq \max\{\mu^{C_4}_{A}, \mu^{C_4}_{B}\}$. 
\end{restatable}
\begin{figure}[ht!]
    \centering
    \includegraphics[width=0.8\textwidth]{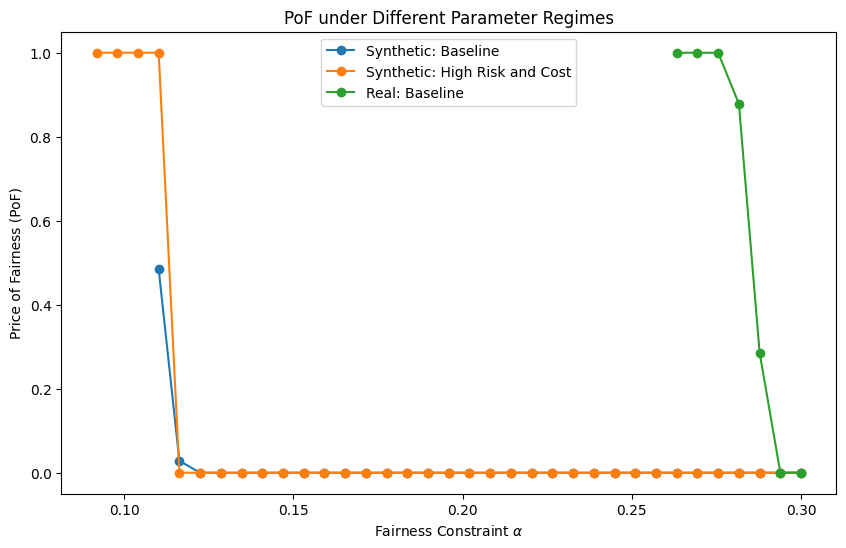}
    \caption{\label{fig:pof} PoF for different datasets and parameters.}
\end{figure}

\vspace{-3mm}
\section{Experiments}
\label{sec:exp}

\smallskip \noindent\textbf{Dataset and Implementation Details.}\footnote{Our code is available on \url{https://github.com/shahin-jabbari/LongTerm-Selection-Fairness}.}
We experimented on two datasets. One is a synthetic dataset with scores drawn from two Gaussian distributions with different means but the same variance, post-processed to integers in 0 to 100. 
The other is a real-world dataset of FICO credit scores~\cite{fico}, also used in~\cite{HardtPS16,LiuDRSH18}. We selected individuals labeled as White or Black to form groups $A$ and $B$. We used different parameter choices to define instances, and these choices are described in each experiment. In all experiments, the $p(x)$ increases linearly in the range of scores from 0 to 1. To solve for the optimal fair policy in the single-step, we used an LP solver since the optimization in Equation~\ref{eq:opt-fair} is linear in $\pi$. Motivated by Proposition~\ref{thm:thresh_opt}, we also implemented an algorithm that searches over all possible thresholds for each group. See Appendix~\ref{sec:appendix-exp} for more details.

\smallskip \noindent\textbf{Single-step Results.}
As our first experiment, we study the PoF on our two datasets. For each dataset, we use a baseline set of parameters to generate an instance. For the synthetic experiments, we also generate an instance with high cost and risk, which, compared to the baseline, has a higher $U^-$ and $C^-$ but the same $U^+$ and $C^+$. All instances satisfy Assumption~\ref{assumpt:struc_assumption}. See Appendix~\ref{sec:appendix-exp} for more details. Figure~\ref{fig:pof} plots the PoF for each choice as a function of $\alpha$ times the range of the score. So $\alpha=0$ corresponds to enforcing exact mean-equality and $\alpha=1$ corresponds to no constraints on the means. If for any given $\alpha$, no fair policy is found, there are no corresponding points on the curves. We observe that a lower $\alpha$ is required for synthetic data compared to real data to achieve 0 PoF. In Appendix~\ref{sec:appendix-exp-sensitivity}, we show that the simple thresholding algorithm is also effective when Assumption~\ref{assumpt:struc_assumption} is violated.

\begin{figure*}[ht!]
    \centering
    \begin{subfigure}[b]{0.48\textwidth}
        \centering
        \includegraphics[width=\textwidth]{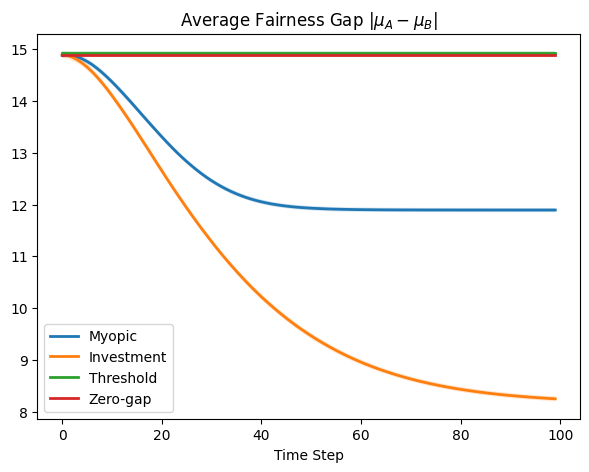}
        \caption{The average gap between group means over time\label{fig:multi-gap}}
    \end{subfigure}
    \begin{subfigure}[b]{0.49\textwidth}
    \centering
        \includegraphics[width=\textwidth]{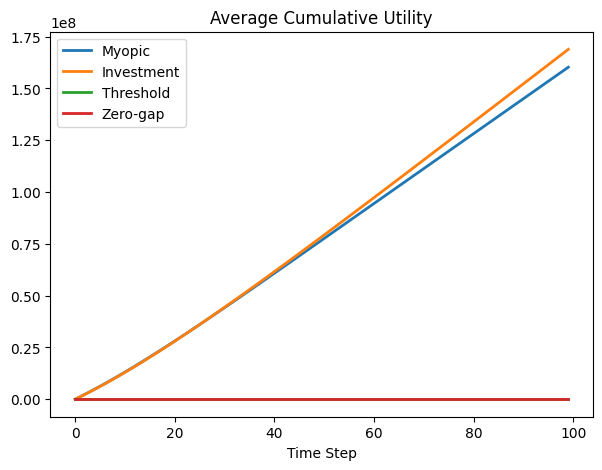}
        \caption{The avarege utility over time\label{fig:multi-utility}}
    \end{subfigure}
    \caption{The Multi-step experiments: left and right panels track the average fairness and utility of each of the policies over time.\label{fig:multi-step}}
\end{figure*}

\smallskip \noindent\textbf{Multi-step Results.}
We investigate the long-term dynamics of fairness and utility in a sequential selection process, modeled as a feedback loop where algorithmic decisions influence future population states. For this experiment, we sample a population of $N=1,000,000$ agents divided from the baseline synthetic distribution and analyze how the scores evolve over a horizon of $T=100$ steps under different policies. At each time step, a policy selects a subset of agents, who then succeed or fail.
We evaluate four distinct policies: a \emph{myopic} policy maximizing immediate utility, our \emph{investment} policy maximizing expected score growth, a {threshold} policy constrained by a relaxed fairness gap ($\alpha=0.01$), and a \emph{zero-gap LP} policy enforcing strict equality of expected future means via linear programming. The myopic and investment policies prioritize aggregate utility and capability growth, respectively, without explicit distributional constraints. In contrast, the threshold policy introduces a soft constraint, optimizing utility while bounding the expected inter-group mean score difference. The zero-gap LP approach represents a \emph{hard} constraint, solving for selection probabilities $\pi_g(x)$ that strictly force the expected fairness gap to zero at each step. These policies allow for an analysis of the trade-offs between maximizing total utility and achieving fairness in the multi-step setting.

Results are summarized in Figure~\ref{fig:multi-step}: left and right panels track the difference between the mean of the groups and total utility obtained by each of the policies, averaged over 5 runs.\footnote{The error bars are included, but they are so small that they are invisible due to the large population size. See Appendix~\ref{sec:appendix-exp-multi-step} for error bars on a smaller population.} We observe that our proposed investment policy is not only effective in reducing unfairness but also achieves higher utility compared to all the baselines. The experiments also show that our theory is fairly robust, as the $p(x)$ distribution does not satisfy Assumption~\ref{assump:exponential-decrease-of-fail-prob}, yet the empirical PoF is still small.

\vspace{-2mm}
\section{Discussion and Future Work}
\label{sec:discussion}
A central contribution of our framework is reframing the price of fairness (PoF) from a static penalty to a dynamic investment. Whereas prior work (e.g.,~\cite{Bertsimas11}) views PoF as an inherent utility loss, our multi-step analysis shows that this loss is largely transient, i.e., a \emph{down payment} that rehabilitates the score distribution of the disadvantaged group. Our theorems formalize the resulting \emph{patience premium}, indicating that the apparent conflict between fairness and utility is partly driven by short horizons: under long-term greedy policies, fairness constraints eventually align with profit maximization, provided the allocator can absorb the initial loss. Our fairness notion minimizes gaps between group means. While analytically convenient, such mean-based metrics capture only average outcomes and do not ensure individual-level fairness, nor do they preclude \emph{fairness gerrymandering}.

There are several directions to extend our work. First, individuals are often represented by high-dimensional features with unobservable attributes, and deployable policies may be significantly more complex. Moreover, feature changes may arise not only from practice but also from improvement~\cite{KleinbergR19,HeidariNG19} or manipulation and gaming~\cite{HardtMPW16,DongRSWW18}. Second, stronger notions of fairness could require closeness of entire distributions (e.g., via total variation or earth mover’s distance) or enforce fairness at every time step rather than only at the end~\cite{JabbariJKMR17}.
Third, we assume the decision-maker can use sensitive attributes, although optimal fair threshold policies may be group-specific when feature distributions differ~\cite{Corbett-DaviesPF+17}. In practice, sensitive attributes may be unavailable or prohibited~\cite{CelisMV21}, and populations often consist of multiple, intersecting protected groups~\cite{KearnsNRW18,Hebert-JohnsonKRR18}. Finally, many applications face resource constraints (e.g., college admissions) not captured by our model. We leave these questions for future work.
\section*{Acknowledgement}
We thank Alvand Vahediahmar for discussion on the earlier stages of this work. The authors used Gemini Pro as an aid in writing code. All
LLM-produced content was reviewed and edited by the authors.

\bibliographystyle{plainnat}
\bibliography{reference}

\appendix
\section{Omitted Details from Section~\ref{sec:single-step}}
\label{sec:appendix-opt-fair-character}

\subsection{Proof of Observation~\ref{obs:no-exploitative} (No $C_2$ category)}
\begin{proof}
We prove this by showing that the probability threshold required for positive utility is strictly higher than the threshold required for positive score change.

Recall that for an individual with score $x$, the expected utility is $\mathbb{E}[u(x)] = p(x)U^+ + (1-p(x))U^-$ and the expected score change is $\mathbb{E}[\Delta(x)] = p(x)C^+ + (1-p(x))C^-$.
An individual qualifies for $C_2$ if and only if they satisfy the following two inequalities regarding their success probability $p(x)$:
\begin{align*}
    & p(x) \ge \frac{-U^-}{U^+ - U^-};  \tag{Profitability}\\
    & p(x) < \frac{-C^-}{C^+ - C^-}. \tag{Score Degradation}
\end{align*}
However, Assumption 2 states that $U^+/U^- > C^+/C^-$. Since $U^-$ and $C^-$ are negative, we can manipulate this inequality to compare the thresholds:
\begin{align*}
    \frac{U^+}{U^-} > \frac{C^+}{C^-} &\implies \frac{U^+}{U^-} - 1 > \frac{C^+}{C^-} - 1 \\
    &\implies \frac{U^+ - U^-}{U^-} > \frac{C^+ - C^-}{C^-} \\
    &\implies \frac{U^-}{U^+ - U^-} < \frac{C^-}{C^+ - C^-} \tag{Taking the inverse flips the inequality} \\
    &\implies \frac{-U^-}{U^+ - U^-} > \frac{-C^-}{C^+ - C^-}. \tag{Multiplying by $-1$ flips the inequality back})
\end{align*}
Since $\tau_u = \frac{-U^-}{U^+ - U^-}$ is the profitability threshold and $\tau_{\Delta} = \frac{-C^-}{C^+ - C^-}$ is the score maintenance threshold, the derivation above confirms that $\tau_u > \tau_{\Delta}$. Consequently, it is impossible for any $p(x)$ to simultaneously satisfy $p(x) \ge \tau_u$ and $p(x) < \tau_{\Delta}$, which implies that $C_2$ is empty.
\end{proof}

\subsection{Proof of Proposition~\ref{pro:opt}}
\begin{proof}
Definition~\ref{def:score-category} implies that category 1 individuals are the only set that has non-negative expected utility. Hence, excluding them will decrease the total utility. Similarly, including individuals from other categories can only decrease the total utility. Assumption~\ref{assumpt:p} and the definition of expected utility implies that $\EE[u(x)]\geq 0$ iff $p(x) \geq -\uminus/(\uplus-\uminus)$.
\end{proof}

\subsection{Proof of Theorem~\ref{thm:thresh_opt}}
We first give the restatement of Theorem~\ref{thm:thresh_opt} before showing the proof.

\TheoremThresholdOPT*
\begin{proof}
Suppose the optimal fair policy $\pi$ is not a threshold. We show that there exists a threshold policy $\pi'$ that achieves at least the utility of $\pi$ while also satisfying the fairness constraint. To simplify exposition, we assume for all $ x$ that $ D_g(x) > 0$ for at least one group. If this is not the case, we can remove $x$ from $X$ without affecting the utility or the fairness constraint. If $\pi$ becomes a threshold policy after this change, we can easily create $\pi'$ from $\pi$ by setting $\pi'(x)$ for all removed $ x$'s to be either 0 or 1, depending on the position of $x$ and the threshold.

Since $\pi$ is not a threshold policy, there exists a group $g\in \{A, B\}$ and scores $x_1$ and $x_2$ such that $x_1 > x_2$, $\pi_g(x_1) < 1$, and $\pi_g(x_2) > 0$. We consider two cases based on the value of $\mu'_g(\pi_g)$ for the non-thresholding group $g$ and the other group $g'$: (1) $\mu'_g(\pi_g) - \mu'_{g'}(\pi_{g'}) \in [-\alpha, \alpha)$ and (2) $\mu'_g(\pi_g) = \mu'_{g'}(\pi_{g'}) + \alpha$.

\paragraph{Case (1):}
In case (1), we assume $\mu'_g(\pi_g) - \mu'_{g'}(\pi_{g'}) \in [-\alpha, \alpha)$. We consider policy $\pi'_g$ that is the same as $\pi_g$ for all $x \in X \setminus \{x_1, x_2\}$, $\pi'_g(x_1) = \pi_g(x_1) + \epsilon_1$ and $\pi'_g(x_2) = \pi_g(x_2) - \epsilon_2$ for sufficiently small $\epsilon_1, \epsilon_2 \geq 0$ satisfying $\epsilon_1 D_g(x_1) = \epsilon_2 D_g(x_2).$ In words, $\pi$' selects a higher fraction of individuals with the higher score of $x_1$ and a lower fraction of individuals with the lower score $x_2$ compared to $\pi$. We show $\pi'$ has a higher expected utility than $\pi$ and also satisfies fairness.

Note that the difference between the expected utility of $\pi'$ and $\pi$ can be written as $\epsilon_1 D_g(x_1) \EE\left[u(x_1)\right]-\epsilon_1 D_g(x_2)\EE\left[u(x_2)\right]$.  We show that this number is non-negative. By assumption $\epsilon_1 D_g(x_1) = \epsilon_2 D_g(x_2)$ in the construction of $\pi'$ and the definition of expected utility, this difference can be written as 
\begin{align*}
\epsilon_1 D_g(x_1) &\left(\EE\left[u(x_1)\right]-\EE\left[u(x_2)\right]\right) = \\
&\epsilon_1 D_g(x_1) \left(p(x_1) - p(x_2)\right)\left(\uplus - \uminus\right),
\end{align*}
which is positive since $p$ is monotone by Assumption~\ref{assumpt:p} and $\uplus - \uminus$ is positive.

It suffices to show $\pi'$ satisfies fairness. By assumption $\epsilon_1 D_g(x_1) = \epsilon_2 D_g(x_2)$ and the definition of expected change in score,
\begin{align*}
\mu'_g(\pi')-\mu'_g(\pi) 
&= \epsilon_1 D_g(x_1) \left((p(x_1) - p(x_2)\right)\left(\cplus - \cminus\right),
\end{align*}
which is positive since $p$ is monotone by Assumption~\ref{assumpt:p} and $\cplus - \cminus$ is positive. Moreover, this difference can be made arbitrarily small by setting $\epsilon_1$ as close as needed to $0$. Finally, $\delta'(\pi')=\delta'(\pi)+ \mu'_g(\pi')-\mu'_g(\pi)$. Together with the assumption of this case that $\delta'(\pi) < \alpha$ this implies $\delta'(\pi') < \alpha$ and hence $\pi'$ is also fair.

\paragraph{Case (2):}
In case (2), we assume $\mu'_g(\pi_g) = \mu'_{g'}(\pi_{g'}) + \alpha$. We consider two subcases: (2a) $\EE\left[\Delta(x_1)\right] \EE\left[\Delta(x_2)\right] > 0$, and (2b) $\EE\left[\Delta(x_1)\right]\EE\left[\Delta(x_2)\right] 
\leq 0$.

\paragraph{Case (2a):}
For case (2a), we consider the policy $\pi'_g$ that is the same as $\pi_g$ for all $x \in X \setminus \{x_1, x_2\}$, $\pi'_g(x_1) = \pi_g(x_1) + \epsilon_1$ and $\pi'_g(x_2) = \pi_g(x_2) - \epsilon_2$ for sufficiently small $\epsilon_1, \epsilon_2 \geq 0$ satisfying
$$\epsilon_1 = \frac{\epsilon_2 D_g(x_2) \EE[\Delta(x_2)]}{D_g(x_1) \EE[\Delta(x_1)]}\epsilon_2.$$

By this construction of $\pi'$ and the definition of expected change in score, the difference between $\mu'_g(\pi')$ and $\mu'_g(\pi)$ can be written as 
\begin{align*}
&\left[\frac{\epsilon_2 D_g(x_2) \EE[\Delta(x_2)]}{D_g(x_1) \EE[\Delta(x_1)]}\epsilon_2\right] D_g(x_1) \EE[\Delta(x_1)] - \epsilon_2 D_g(x_2) \EE[\Delta(x_2)]\\    
&= \epsilon_2 D_g(x_2) \EE[\Delta(x_2)] - \epsilon_2 D_g(x_2) \EE[\Delta(x_2)]=0,\\
\end{align*}
indicating that $\pi'$ also satisfies fairness. 

It remains to show that $\pi'$ achieves a strictly higher expected utility than $\pi$. By the construction assumption of $\pi'$ and the definition of expected utility, the difference between the expected utility of $\pi'$ and $\pi$ can be written as

\begin{align*}
&\epsilon_1 D_g(x_1) \EE[u(x_1)] - \epsilon_2 D_g(x_2) \EE[u(x_2)]\\
&= \frac{\epsilon_2 D_g(x_2) \EE[\Delta(x_2)]}{D_g(x_1) \EE[\Delta(x_1)]}\epsilon_2 D_g(x_1) \EE[u(x_1)] - \epsilon_2 D_g(x_2) \EE[u(x_2)]\\
&= \epsilon_2 D_g(x_2) \left( \frac{\EE[\Delta(x_2)]}{\EE[\Delta(x_1)]} \EE[u(x_1)] - \EE[u(x_2)] \right).
\end{align*}
Note that both $\epsilon_2$ and $D_g(x_2)$ are positive. If we show that $\EE[\Delta(x_2) \EE[u(x_1)]/\EE[\Delta(x_1)] > \EE[u(x_2)]$, we can demonstrate that $\pi'$ has a strictly greater expected utility than $\pi$. Since $\EE[\Delta(x_1)]\EE[\Delta(x_2)]>0$, they should have the same sign. We analyze each case separately.

In the case where both $\EE[\Delta(x_1)]$ and $\EE[\Delta(x_2)]$ are positive, by Assumption~\ref{assumpt:struc_assumption}, $\uplus/\uminus< \cplus/\cminus$. This indicates that
\begin{align*}
\vspace{-10mm}&\frac{\uplus}{\uminus} > \frac{\cplus}{\cminus} \implies \frac{\uplus-\uminus}{\uminus} > \frac{\cplus-\cminus}{\cminus}
\end{align*}

Multiplying both sides by $p(x_1)-p(x_2)$ gives us the following set of derivations. 
\begin{align*}
& \left(p(x_1)-p(x_2)\right)\frac{\uplus-\uminus}{\uminus} > \left(p(x_1)-p(x_2)\right)\frac{\cplus-\cminus}{\cminus} \\
& \implies \EE[u(x_1)]\EE[\Delta(x_2)] > \EE[u(x_2)]\EE[\Delta(x_1)] \\
& \implies \EE[u(x_1)]\frac{\EE[\Delta(x_2) }{\EE[\Delta(x_1)]} - \EE[u(x_2)] > 0,
\end{align*}
as desired. The second line multiplies both sides of the inequality by $p(x_1)-p(x_2)>0$, and the third line is obtained by algebraic simplification.

In case where both $\EE[\Delta(x_1)]$ and $\EE[\Delta(x_2)]$ are negative, the ratio $\EE[\Delta(x_2)]/\EE[\Delta(x_1)]$ is positive and greater than 1 since the expected score change is a monotone function of score. Therefore, 
\begin{align*}
\frac{\EE[\Delta(x_2)]}{\EE[\Delta(x_1)]} \EE[u(x_1)] - \EE[u(x_2)] \geq 
\EE[u(x_1)] - \EE[u(x_2)] > 0.
\end{align*}

\paragraph{Case (2b):}
In case (2b), we have that $\EE[\Delta(x_1)] \EE[\Delta(x_2)] < 0$. Since $\EE[\Delta(x)]$ is a monotone function of the score by Assumption~\ref{assumpt:p}, then it must be that $\EE[\Delta(x_1)] \geq 0$ and $\EE[\Delta(x_2)] < 0$. Consider two subcases: (2bi) $\EE[u(x_1)] < 0$, and (2bii) $\EE[u(x_1)] \geq 0$.

In case (2bi), where $\EE[u(x_1)] < 0$, construct $\pi'_g$ such that $\pi'_g(x) = \pi_g(x)$ for all  $x \in X \setminus \{x_2\}$ and $\pi_g'(x_2) = \pi_g(x_2) - \epsilon$ for sufficiently small $\epsilon$. Since $\EE[u(x_1)] > \EE[u(x_2)]$ by Assumption~\ref{assumpt:p} and $\EE[u(x_1)] < 0$, it follows that $\EE[u(x_2)] < 0$. Reducing $\pi_g'(x_2)$ by $\epsilon$ compared to $\pi_g(x_2)$ ensures that $\pi'$ remains fair since $\mu'_{\pi'_g} < \mu'_{\pi_g}$. Moreover, $\pi_g'$ has a higher expected utility than $\pi_g$ since $\EE[u(x_2)] < 0$. 

In case (2bii), where $\EE[u(x_1)] > 0$, we have $\EE[u(x_2)] \leq 0$ because if $\EE[u(x_2)] > 0$, then by Assumption~\ref{assumpt:struc_assumption}, we would have $\EE[\Delta(x_2)] > 0$, which is a contradiction. If $\EE[u(x_2)] < 0$ with the same argument as (2bi), we can create a fair $\pi'_g$ with strictly better utility by setting $\pi'_g(x_2) = \pi_g(x_2) - \epsilon$. Note that $\pi'_g$ is fair because $\mu'_{\pi'_g} < \mu'_{\pi_g}$. Finally, when
$\EE[u(x_2)] = 0$ then $p(x_2) = \tfrac{-\uminus}{\uplus - \uminus}$ by definition of expected utility. As a result,
\begin{align*}
\EE[\Delta(x_2)] &= p(x_2) \left(\cplus - \cminus\right) + \cminus 
&= \frac{\cminus \uplus - \uminus \cplus}{\uplus - \uminus}.
\end{align*}
The denominator of the fraction is positive, and by Assumption~\ref{assumpt:struc_assumption} the numerator is positive, implying $\EE[\Delta(x_2)] \ge 0$, contradicting the assumption $\EE[\Delta(x_2)] < 0$.
\end{proof}
\subsection{Proof of Theorem~\ref{pro:lb-pof}}
We first give a restatement of \Cref{pro:lb-pof} before giving the proof.
\ThmPoFGeneral* 
\begin{proof}
Consider the following set of instances where $w_A=1/2$, $w_B=1/2$, $D_A[x]=1$, $D_B[x']$ for $x > x' \in [0,1]$. In these instances, each group is evenly represented in the population, and in each group, the distribution is concentrated on individuals with a group-dependent score. We select $\cplus$, $\cminus$, $\uplus$, and $\uminus$ such that Assumption~\ref{assumpt:struc_assumption} is satisfied and the new scores remain in $X$

Since $p$ is a monotone function, we choose $x$ and $x'$ such that $p(x)=1$ and $p(x')=\frac{-\uminus}{\uplus - \uminus}$. This choice implies that $\EE[u(x)]=\uplus > 0$ and $\EE[u(x')]=0$. These choices imply $x$ belongs to category 1 while $x'$ belongs to category 3.

Proposition~\ref{pro:opt} implies that the optimal policy $\pi^*$ selects all individuals with score $x$ in category 1 and achieves an expected utility of $\uplus/2$. On the other hand, if any fair policy $\pi$ exists (depending on $\alpha$), it selects a fraction $\pi_A(x)$ and $\pi_B(x')$ such that 
$$\Big|x+\pi_A(x)\EE[\Delta(x)]+(x'+\pi_B(x')\EE[\Delta(x')]\Big|\leq \alpha.$$
The utility of such a policy is $\pi_A(x)\uplus$.

Setting $\cplus=1$, $\cminus=-\cplus$, $\uplus=1$, $\uminus=-1.1$, $x-x'=\eps$ (for arbitrary small $\eps$) implies that any fair policy should satisfy $\pi_A(x)\leq \alpha-\eps.$ Since the utility of any fair policy is proportional to $\pi_A(x)$, then the optimal fair policy sets $\pi_A(x)$ to 
$$\pi_A(x)=\min\{\max\{0, \alpha-\eps\},1\},$$ where the expression guarantees $\pi_A(x)\in[0,1]$. 

This instance gives a lower bound of $1-(\alpha-\eps)$ on the PoF as claimed.
\end{proof}

\subsection{Proof of Theorem~\ref{thm:pof-tv}}
The original statement of \Cref{thm:pof-tv} is as follows.
\ThmPoFNondegrade*

In what follows, we state \Cref{thm:pof-tv} in a more detailed version (\Cref{pro:pof-tv}) and give the proof.
\begin{proposition}
\label{pro:pof-tv}
For any $\epsilon > 0$, let $I_\epsilon$ denote the set of instances such that $D_{TV}(D_A, D_B) \leq \epsilon$ where $D_{TV}$ is the total variation distance. Then, in the single-step setting, $\max_{I_\epsilon}\left( 1 - \tfrac{\fair^+(I)}{\opt(I)}\right) \geq 1 - O(\alpha)$ where $\fair^+(I)$ is the expected utility of the best non-degrading fair policy.
\end{proposition}
\begin{proof}
The construction of the proof follows a similar template to the construction in the proof of~\Cref{pro:lb-pof}. Again we consider two distributions such that $w_A=w_B=1/2$, $D_A[x]=\epsilon$, $D_A[x'']=1-\epsilon$, $D_B[x']=\epsilon$, and $D_B[x'']=1-\epsilon$ for $x > x' > x'' \in [0,1]$. We set $p(x'')\approx 0$, $p(x')=1/2$, and $p(x)=1$. As such, $x$ belongs to category 1, $x'$ to category 3, and $x''$ to category 4.

By construction, $D_{TV}(D_A, D_B)=\epsilon$ and Assumptions~\ref{assumpt:p}~and~\ref{assumpt:struc_assumption} are also satisfied if we set $\uplus, \uminus, \cplus, $ and $\cminus$ exactly as in~\Cref{pro:lb-pof}. Since the class of fair policies is restricted to be non-degrading, no policy in this class can satisfy $\pi_g(x'')>0$ for $g\in\{A,B\}$ because $x''$ belongs to category 4. This reduces the analysis to that of~\Cref{pro:lb-pof}, which results in the same bound on PoF.
\end{proof}

\subsection{An Upper Bound for the PoF}
Having proved multiple lower bounds, we now characterize sufficient conditions to get low price of fairness as in \Cref{pro:pof-sufficient-condition}.
\begin{proposition}
\label{pro:pof-sufficient-condition}
Suppose, for instance $I$, the following conditions hold for an $\epsilon \in (0, 1/2)$.
\begin{enumerate}
    \item $\sum_{x\in C_1} D_B(x)\EE[u(x)] \geq \eps \cdot \sum_{x\in C_1} D_A(x)\EE[u(x)]$.
    \item \emph{At least one} of the following properties holds:
    \begin{itemize}
        \item Either $\sum_{x\in C_1} D_B(x) \EE[\Delta(x)]\geq \mu_A-\mu_B-\alpha$;
        \item Or there exists a subset of individual $\Ctilde_3\subseteq C_3$ \emph{in group $B$}, such that 
        \begin{align*}
        \card{\sum_{x\in \Ctilde_3}D_B(x)\EE\left[u(x)\right]} \leq \eps\cdot\sum_{x\in C_1}D_B[x]\EE[u(x)],
        \end{align*}
        and
        \[\sum_{x\in C_1 \cup \Ctilde_3} D_B(x) \EE[\Delta(x)] \geq \mu_A-\mu_B-\alpha\].
    \end{itemize}
\end{enumerate}
Then, $\pof = \left( 1 - \tfrac{\fair(I)}{\opt(I)}\right)\leq 1-\eps/4.$
\end{proposition}
\begin{proof}
We show that under such conditions, the policy that selects all individuals in $C_1\cap B$ and the individuals in $\tilde{C}_3$ would be fair, and the utility would make up at most $\eps/2$-fraction of the single-step optimal policy. Since the optimal fair policy is at least as good as this policy, we have an upper bound on the $\pof$.

In particular, the second condition provides two sufficient conditions for which an $\alpha$-fair policy selects all the category 1 individuals from group $B$. If $\sum_{x\in C_1} D_B(x) \EE[\Delta(x)]\geq \mu_A-\mu_B-\alpha$, simply admitting individuals $C_1 \cap B$ would be enough. On the other hand, if the conditions in the second bullet hold, we can admit all individuals from $C_1 \cap B$ and $\Ctilde_3$ to guarantee fairness.

In either way, the utility we could gain is at least $\max\{(1-\eps) \cdot \sum_{x\in C_1} D_B(x)\EE[u(x)], \eps\cdot (1-\eps) \cdot \sum_{x\in C_1} D_A(x)\EE[u(x)]\}$. On the other hand, the utility of the optimal algorithm is at most $\sum_{x\in C_1} D_B(x)\EE[u(x)] + \sum_{x\in C_1} D_A(x)\EE[u(x)]$. Therefore, assuming w.log. that $\sum_{x\in C_1} D_A(x)\EE[u(x)] \geq \sum_{x\in C_1} D_B(x)\EE[u(x)]$ the price of fairness could be bounded by
\begin{align*}
    \pof &\leq 1-\frac{\eps\cdot (1-\eps) \cdot \sum_{x\in C_1} D_B(x)\EE[u(x)]}{2\cdot \sum_{x\in C_1} D_A(x)\EE[u(x)]}\\
    & \leq 1-\eps/4, \tag{using $\eps<1/2$}
\end{align*}
as desired.
\end{proof}

\section{Omitted Details from Section~\ref{sec:repeated}}
\label{sec:appendix-repeated}
We provide the detailed analysis and proofs of the statements in Section~\ref{sec:repeated}.

\subsection{A More Detailed Description of the Multi-step Objective}
\label{subsec:detailed-multi-step-objective}
We start with a more detailed characterization of the multi-step objective function. 
The detailed characterization is based on the \emph{tracking} of score changes over the stochastic process.
Let $\calG(x)^{(t)}, t\in [T]$ be the family of scores starting from $x$ at the $t$-th time step (i.e., $\calG(x)^{(1)}=\{x\}$). 
Note that $\calG(x)^{(t)}$ is a random variable depending on the randomness in each step, and we let $G(x)^{(t)}$ be a certain realization of $\calG(x)^{(t)}$ (we have discussed this in Section~\ref{subsec:multi-step-model-and-notation}).
Furthermore, if we fix all $\calG(x)^{(t)}$ for $x\in X$, we could define 
$\Dtilde^{(t)}(x):=\frac{\card{\{y\in \calG(x)^{(t)}\mid y=x^{t}\}}}{\card{\calG(x)^{(t)}}}$
as the fraction of individuals with score $x^{(t)}$ \emph{in $\calG(x)^{(t)}$}.
This is important since we want to insist on $\sum_{x^{(t)}\in \calG(x)^{(t)}}\Dtilde^{(t)}(x) = \Dtilde(x)$.
For randomness in step $t$, we use $\EE_{(1:t-1)}$ and $\Pr_{(1:t-1)}$ to denote the expectation and probability induced by the randomness of the first $t-1$ steps, and $\EE_{t}$ and $\Pr_{t}$ to denote the same notion with the randomness of the $t$-th step (conditioning on certain realizations of the first $t-1$ steps). The analogous $D_{g}^{(t)}(x)$ and $\pi^{(t)}_g(x^{(t)})$ for $g\in \{A,B\}$ can be defined by focusing on individuals from group $g$. The objective function in Equation~\ref{equ:multi-step-utility-simplified} can be written in the alternative form as in Equation~\ref{equ:multi-step-utility}. 
Note that in Equation~\ref{equ:multi-step-utility}, we wrote $x^{(t)}\in G(x)^{(t)}$ simply as $x^{(t)}$ for simplicity of the notation.

\begin{figure*}[htb!]
\begin{equation}
\label{equ:multi-step-utility}
\begin{aligned}
& \max_{\Pi_{T}} \expect{\sum_{t=1}^{T} \sum_{g \in {A,B}} w_g \sum_{x \in X} D^{(t)}_g(x) \pi^{(t)}_g(x) u(x)}\\
& = \max_{\Pi_{T}} \sum_{g\in \{A, B\}}\expect{w_g\sum_{t=1}^{T}\sum_{x\in X} \sum_{x^{t}\in \calG(x)^{(t)}}\pi^{t}_g(x^{t})\cdot \Dtilde^{(t)}_g(x^{t})\cdot u(x^{t})} \\
& = \max_{\{\pi^{(t)}\}_{t=1}^{T}} \sum_{g\in \{A, B\}}w_g\sum_{x\in X} \sum_{t=1}^{T}  \paren{\EE_{t}\bracket{\sum_{x^{(t)}} \Dtilde^{(t)}_g(x) \cdot u(x^{(t)})\cdot \pi^{t}_g(x^{(t)}) \mid G(x)^{(t)}} \cdot \Pr_{{(1:t-1)}}\paren{\calG(x)^{(t)}=G(x)^{(t)}}}.
\end{aligned}
\end{equation}
\end{figure*}

From Equation~\ref{equ:multi-step-utility-simplified}, the correlation between the steps is more obvious. Here, $\calG(x)^{(t)}$ is strongly dependent on the realizations of previous steps.

\subsection{Justification of Assumptions~\ref{assump:low-failure-prob-C1-C3} and \ref{assump:exponential-decrease-of-fail-prob}}
\label{subsec:assump-justification}
We did not provide justifications for Assumptions~~\ref{assump:low-failure-prob-C1-C3} and \ref{assump:exponential-decrease-of-fail-prob} due to space limits, and we provide the justification here. 

\paragraph{Justification of Assumptions~\ref{assump:low-failure-prob-C1-C3}.} Assumption~\ref{assump:low-failure-prob-C1-C3} is essentially an assumption about ``stability''. Here, the failure probability scales inversely with $\xmax$, which directly controls the variance of the random walk. Furthermore, we can tolerate the failure probability up to $\beta$ since $\beta$ is the initial gap between the average scores of the advantaged and disadvantaged groups. A failure probability that is way more than $\beta$ would effectively ``cancel'' the advantage of group $A$. In real world, this assumption corresponds to a scenario when individuals in $C_1 \cup C_3$ are relatively trustworthy and are expected to pay back.

\paragraph{Justification of Assumptions~\ref{assump:exponential-decrease-of-fail-prob}.} We use a very natural family of probability distribution, namely the logistic distribution, to justify the assumption. 
A typical logistic distribution is a monotone function with characterized as follows:
\begin{align*}
    p(x) = \frac{1}{Z}\cdot \frac{1}{1+\exp(-kx)},
\end{align*}
where $k$ is a parameter and $Z$ is a normalization factor. Using this formulation, we have
\begin{align*}
    1-p(x) = \frac{1}{Z}\cdot \frac{\exp\paren{-kx}}{1+\exp\paren{-kx}}\geq \frac{\exp\paren{-kx}}{2Z},
\end{align*}
where the inequality holds since $-kx\leq 0$.
On the other hand, we have 
\begin{align*}
    1-p(x+C^+) &= \frac{1}{Z}\cdot \frac{\exp\paren{-kx-kC^+}}{1+\exp\paren{-kx-kC^+}}\\
    &\leq \frac{1}{Z}\cdot \exp\paren{-kx-kC^+} \tag{$\exp(\cdot)$ is a non-negative function}\\
    &= \frac{1}{Z}\cdot \exp\paren{-kx}\exp\paren{-kC^+}.
\end{align*}
Therefore, as long as $\exp\paren{-kC^+} \leq \frac{1}{6}$, the inequality will hold. For $k=1$, this is true for all $C^+\geq \ln(6)$, which is around $1.79$.

\subsection{Technical Preliminaries}
\label{app:tech-prelim}
We introduce some technical tools we used in the analysis. We start with the standard concentration inequalities, including Markov bound and Chernoff-Hoeffding bound. 

\begin{proposition}[Chernoff-Hoeffding bound]
\label{prop:chernoff}
	Let $X_1,\ldots,X_m$ be $m$ independent random variables with support in $[0,1]$. Define $X := \sum_{i=1}^{m} X_i$. Then, for every $t > 0$, 
	\begin{align*}
		\Pr\paren{\card{X - \expect{X}} > t} \leq 2 \cdot \exp\paren{-\frac{2t^2}{m}}. 
	\end{align*}
\end{proposition}

More generally, for random variables $\{X_i\}_{i=1}^{m}$ supported on $\{[a_i, b_i]\}_{i=1}^{m}$, we can also write the concentration of the \emph{average value} as follows.

\begin{proposition}[Hoeffding bound beyond 0-1 variables]
\label{prop:hoeffding-general}
Let $X_1,\ldots,X_m$ be $m$ independent random variables with support in $[a_i,b_i]$ for $X_i$. Define $\bar{X} := \frac{1}{m}\cdot \sum_{i=1}^{m} X_i$.
\begin{align*}
    P\left( |\bar{X} - \expect{\bar{X}} | \ge \epsilon \right) \le 2 \exp \left( - \frac{2m^2 \epsilon^2}{\sum_{i=1}^m (b_i - a_i)^2} \right)
\end{align*}
\end{proposition}

Next, we state the property of Bernstein's inequality for the summation of the random variables.

\begin{proposition}[Bernstein's inequality]
\label{prop:bernstein}
Let $X_1,\ldots,X_m$ be $m$ independent random variables such that $\max_i \card{X_i}\leq M$ almost surely. Then, for every $t > 0$, 
\begin{align*}
    P\left( \left| \sum_{i=1}^n X_i \right| \ge t \right) \le 2 \exp\left( -\frac{t^2/2}{\sum_{i=1}^n E[X_i^2] + Mt/3} \right)
\end{align*}
\end{proposition}

\subsection{The Detailed Proofs}
\subsubsection{Proof of Lemma~\ref{lem:multi-step-variance-control} (The Structural Lemma)}
\begin{proof}
We prove the lemma by iteratively conditioning on the high-probability events for the individuals to succeed. We define the set of scores with increments compared to the previous time step as follows.
\begin{align*}
    \mathcal{X}^{(t)}_{\text{increase}} = \cup_{x} \{x^{(t)}\in G(x)^{(t)}\mid x^{(t)}\geq x^{(t-1)}\}.
\end{align*}
Here, we slightly abuse the notation to let $x^{(t)}$ and $x^{(t-1)}$ be the score of a certain individual so that the increment is well-defined.
By our investment policy, we have $\mathcal{X}^{(t)}_{\text{increase}} \subseteq \mathcal{X}^{(t-1)}_{\text{increase}}\subseteq \cdots \subseteq \mathcal{X}^{(1)}_{\text{increase}}$, where $\mathcal{X}^{(1)}_{\text{increase}}$ contains all scores.
In other words, for the set of scores in $\mathcal{X}^{(t)}_{\text{increase}}$, the random process will always increase the score and the probability $p(x)$ of success. 
Therefore, if we define the random variable
\begin{align*}
    \Delta(x)^{(t)}=
    \begin{cases}
        C^+, \text{with probability $1-\pfail$}\\
        C^-, \text{with probability $\pfail$}
    \end{cases}.
\end{align*}
We have that $\Delta(x)^{(t)}\geq \Delta(x^{(t)})$ for $x^{(t)}\in \mathcal{X}^{(t)}_{\text{increase}}$. 
We work with the random variable $\Delta(x)^{(t)}$ to avoid dependency.

We first show \emph{inductively} that the fraction of the scores we account for is at most $600 \cdot (\frac{1}{2})^{t-1}\cdot \pfail$ fraction with probability at least $199/200$ for the $t$-th time step.
For the $t$-th step, define $\overline{\calX^{(t)}}$ as the set of scores that are not increased in this round, and we have $\expect{\card{\overline{\calX^{(t)}}}}\leq \pfail \cdot  (\frac{1}{3})^{t-1} \card{\mathcal{X}^{(t)}_{\text{increase}}}$ by applying Assumption~\ref{assump:exponential-decrease-of-fail-prob}. 
To elaborate, we can $(1-p(\xmin^{(t)}))$ as the probability for score decrement for the \emph{minimum score} among scores in $\mathcal{X}^{(t)}_{\text{increase}}$. Using this notation, we have $(1-p(\xmin^{(1)}))=\pfail$. Therefore, we can write
\begin{align*}
    \expect{\card{\overline{\calX^{(t)}}}} & = \sum_{x\in \mathcal{X}^{(t)}_{\text{increase}}} (1-p(x)) \tag{by definition}\\
    & \leq (1-p(\xmin^{(t)})) \cdot \card{\mathcal{X}^{(t)}_{\text{increase}}} \tag{using $p(x)$ as a monotone function}\\
    &= (1-p(\xmin^{(t-1)}+C^+)) \cdot \card{\mathcal{X}^{(t)}_{\text{increase}}} \tag{by the dynamics of score changing}\\
    &\leq \frac{1}{3}\cdot (1-p(\xmin^{(t-1)})) \tag{by Assumption~\ref{assump:exponential-decrease-of-fail-prob}}\\
    &\leq (\frac{1}{3})^2 \cdot (1-p(\xmin^{(t-2)}))\cdot \card{\mathcal{X}^{(t)}_{\text{increase}}} \tag{recursively applying the above calculation}\\
    &\leq \cdots \cdots\\
    & \leq (\frac{1}{3})^{t-1} \cdot (1-p(\xmin^{(1)}))\cdot \card{\mathcal{X}^{(t)}_{\text{increase}}} \\
    &= (\frac{1}{3})^{t-1} \cdot \pfail \cdot \card{\mathcal{X}^{(t)}_{\text{increase}}},
\end{align*}
as claimed.
Therefore, by applying the Markov inequality, we have
\begin{align*}
    & \Pr\paren{\card{\overline{\calX^{(t)}}}\geq 600 \cdot (\frac{1}{2})^{t-1}\cdot \pfail\cdot \card{\mathcal{X}^{(t)}_{\text{increase}}}}\\
    & \leq \frac{1}{600}\cdot (\frac{1/3}{1/2})^{t-1} = \frac{1}{600}\cdot (\frac{2}{3})^{t-1}.
\end{align*}
Therefore, we could apply a union bound and show that
\begin{align*}
& \Pr\paren{\card{\overline{\calX^{(t)}}}< 600 \cdot (\frac{1}{2})^{t-1}\cdot \pfail\cdot \card{\mathcal{X}^{(t)}_{\text{increase}}} \text{ for all $t$}}\\ 
& \geq 1-\sum_{t=1}^{\infty} \frac{1}{600}\cdot (\frac{2}{3})^{t-1} \geq 1-\frac{3}{600} = 199/200.
\end{align*}
We condition on this high-probability event for the rest of the proof. We now analyze the deviation of the scores for $x^{(t)}\in \calX^{(t)}_{\text{increase}}$. 
Note that for all such scores, we have $x^{(t)}\geq x+\sum_{\tau=1}^{t}\Delta(x)^{(\tau)}$.
For any $\tau$, we have $M=(C^+-C-)$ and $\expect{(\Delta(x)^{(\tau)})^2}\leq O\paren{(C^{+}-C^{-})^2}$. 
Therefore, by Bernstein's inequality, we have that
\begin{align*}
 & \Pr\paren{x^{(t)}\leq \expect{x^{(t)}}-\sqrt{t} \cdot (C^{+}-C^{-})\polylog(t\cdot \card{C_1 \cup C_3})}\\
 & \leq \exp\paren{-\frac{\Otilde\paren{\sqrt{t} \cdot (C^{+}-C^{-})}^2}{(C^{+}-C^{-})^2\cdot t+\Otilde\paren{\sqrt{t}\cdot (C^{+}-C^{-})^2}}} \\
 & \leq \frac{1}{200 t\cdot \card{C_1 \cup C_3}},
\end{align*}
where $\Otilde(\cdot)$ hides $\polylog(t\cdot \card{C_1 \cup C_3})$ terms.
Therefore, we can apply a union bound over the possible number of scores in $C_1 \cup C_3$ and the $t$ steps to obtain that with probability at least $199/200$, there is 
\begin{align*}
 x^{(t)} &\geq \expect{x^{(t)}}-\sqrt{t} \cdot (C^{+}-C^{-})\polylog(t\cdot \card{C_1 \cup C_3})\\
 &\geq t\cdot C^{+} - \sqrt{t} \cdot (C^{+}-C^{-})\polylog(t\cdot \card{C_1 \cup C_3}),
\end{align*}
if we do \emph{not} care about the maximum scores.
The second inequality is true since we have already conditioned on the set of $\mathcal{X}^{(t)}_{\text{increase}}$.
Therefore, as $t\rightarrow \infty$, the scores in $\mathcal{X}^{(t)}_{\text{increase}}$ would eventually reach $\xmax$, and never decrease due to Assumption~\ref{assump:exponential-decrease-of-fail-prob}. In other words, with probability at least $1-1/200-1/200=99/100$, all but $O(\pfail)$ fraction of scores would fail to reach $\xmax$, as desired.
\end{proof}

\subsubsection{Proof of~\texorpdfstring{\Cref{prop:muilti-step-single-opportunity}}{multi-step-single-opp}}
We give the (re-)statement and proof of \Cref{prop:muilti-step-single-opportunity} as follows.
\ThmMultiStepSingleOpportunity*

\begin{proof}
    At a high level, the statement is true because eventually, every individual who is given the allocation would reach the maximum score. This would mean that the remaining gap between the two groups is restricted to the individuals with scores in $C_4$. Due to Assumption~\ref{assump:no-C4-in-optimal-policy}, the $\pof$ goes to $0$ in the limit as long as the fairness restriction is less restrictive than the gap in category $4$.

    We now formalize the above intuition. We write the expression in Assumption~\ref{assump:real-advantage} in the form of 
    \begin{align*}
        & \paren{\frac{\card{A\cap C_1}}{\card{A}}\cdot \mu^{C_1}_{A} + \frac{\card{A\cap C_3}}{\card{A}}\cdot \mu^{C_3}_{A}} \\
        & \quad - \paren{\frac{\card{B\cap C_1}}{\card{B}}\cdot \mu^{C_1}_{B} + \frac{\card{B\cap C_3}}{\card{B}}\cdot \mu^{C_3}_{B}}> \beta.
    \end{align*}
    Let $\delta=\card{\mu_{A}-\mu_{B}}$ be the initial gap between the mean scores, we write it as a summation of conditional terms:
    \begin{align*}
        \delta & = \card{\mu_{A}-\mu_{B}}\\
        &= \mu_{A}-\mu_{B} \tag{initially group $A$ has a higher score}\\
        &= \frac{\card{A\cap C_1}}{\card{A}}\cdot \mu^{C_1}_{A} - \frac{\card{B\cap C_1}}{\card{B}}\cdot \mu^{C_1}_{B} \\
        & \quad\quad + \frac{\card{A\cap C_3}}{\card{A}}\cdot \mu^{C_3}_{A} - \frac{\card{B\cap C_3}}{\card{B}}\cdot \mu^{C_3}_{B}\\
        &\quad\quad + \frac{\card{A\cap C_4}}{\card{A}}\cdot \mu^{C_4}_{A} -  \frac{\card{B\cap C_4}}{\card{B}}\cdot \mu^{C_4}_{B}.
    \end{align*}
    For any $x\in C_1 \cup C_3$, since the long-term investment policy selects $x$ at each step, and $u(x)$ and $\Delta(x)$ are monotone functions of $x$, we have that $\lim_{T\rightarrow \infty} \EE[\calG(x)^{(T)}] = \{\xmax\}$ for all but $O(\pfail)\cdot \card{C_1 \cup C_3}$ of the scores by Lemma~\ref{lem:multi-step-variance-control}. 
    Define $X_{\text{decrease}}$ as the set of scores that ever decrease during the process, and $\delta_{\text{decrease}}$ as the gap between the scores induced by the scores in $X_{\text{decrease}}$.
    Note that for each score $x\in X_{\text{decrease}}$, the difference in score it could induce is at most $\xmax$ (the scores cannot be negative). 
    Therefore, the total score gaps induced by the scores in $X_{\text{decrease}}$ are at most $\frac{\card{X_{\text{decrease}}}}{\card{X}}\cdot \xmax$, which means
    \begin{align*}
        \delta'(T) \leq \frac{\card{A\cap C_4}}{\card{A}}\cdot \mu^{C_4}_{A} -  \frac{\card{B\cap C_4}}{\card{B}}\cdot \mu^{C_4}_{B} + \frac{\card{X_{\text{decrease}}}}{\card{X}}\cdot \xmax.
    \end{align*}
    Conditioning on the high-probability event of \Cref{lem:multi-step-variance-control}, which implies $\frac{\card{X_{\text{decrease}}}}{\card{X}} \leq O(\pfail)$, we could bound the change of the gaps between scores with
    \begin{align*}
    \lim_{T\rightarrow \infty}\delta-\delta'(T) & \geq \frac{\card{A\cap C_1}}{\card{A}}\cdot \mu^{C_1}_{A} - \frac{\card{B\cap C_1}}{\card{B}}\cdot \mu^{C_1}_{B} \\
    & + \frac{\card{A\cap C_3}}{\card{A}}\cdot \mu^{C_3}_{A} - \frac{\card{B\cap C_3}}{B}\cdot \mu^{C_3}_{B} \\
    &  - O(\pfail)\cdot \xmax\\
    & \geq \beta - \frac{\beta}{N \xmax}\cdot \xmax \tag{applying Assumptions~\ref{assump:real-advantage} and \ref{assump:low-failure-prob-C1-C3}}\\ 
    & >\beta/2.
    \end{align*}
     Rearranging the terms gives us the desired statement.

    For the second part of the statement, note that after the scores in $(C_1 \cup C_3) \setminus X_{\text{decrease}}$ reach $\xmax$, we have 
    \begin{align*}
    \delta'(T) & \leq \frac{\card{A\cap C_4}}{\card{A}}\cdot \mu^{C_4}_{A} -  \frac{\card{B\cap C_4}}{\card{B}}\cdot \mu^{C_4}_{B} + \frac{\card{X_{\text{decrease}}}}{\card{X}}\cdot \xmax \\
    & \leq \max\{\mu^{4}_{A}, \mu^{4}_{B}\}+\beta,
    \end{align*}
    where the last inequality holds due to assumption~\ref{assump:low-failure-prob-C1-C3}.
    Therefore, as long as $\alpha\geq \max\{\mu^{C_4}_{A}, \mu^{C_4}_{B}\}+\beta$, the long-term investment policy is already fair. 
    For the optimal (potentially unfair) policy $\Pi^*$, by Assumption~\ref{assump:no-C4-in-optimal-policy}, it would also aim to bring all scores in $C_1$ and $C_3$ to $\xmax$. Therefore, let $T_0$ be the first step for all scores in $C_1$ and $C_3$ to achieve $\xmax$ by the investment policy. Before $T_0$, on each day, the loss of the utility is at most $(U^{+}-U^{-})\cdot \card{A\cup B}$. After $T_0$, there is no gap between the utility of policies, and we let the daily utility gain be $U^*$. Therefore, we have $\fair(I)\geq (T-T_0)\cdot U^*\cdot (1-O(\pfail))\cdot \card{C_1 \cup C_3}- T_0\cdot (U^{+}-U^{-})\cdot \card{C_1 \cup C_3}$ and $\opt(I)\leq T\cdot U^* \cdot \card{C_1 \cup C_3}$, which implies
    \begin{align*}
        & \lim_{T\rightarrow \infty} \pof \\
        & \leq \lim_{T\rightarrow \infty} 1 - \frac{\paren{(T-T_0)\cdot U^*\cdot (1-O(\pfail))- T_0\cdot (U^{+}-U^{-})}}{T\cdot U^*}\\
        & = \lim_{T\rightarrow \infty} 1 - \frac{T \cdot (1-O(\pfail)) \cdot U^*}{T\cdot U^*} \tag{$T_0$, $U^+$, $U^-$ are finite} \\
        &=O(\pfail), 
    \end{align*}
    which is as desired.
\end{proof}

\subsubsection{Proof of~\texorpdfstring{\Cref{prop:muilti-step-multi-opportunity}}{multi-step-multi-opp}}
We give the (re-)statement and proof of \Cref{prop:muilti-step-multi-opportunity} as follows.
\ThmMultiStepMultiOpportunity*

\begin{proof}[Proof of~\Cref{prop:muilti-step-multi-opportunity}]
For each score $x^{(t)}$, we prove the concentration of the change of scores due to the multiple opportunities provided to each individual.
By Hoeffding bound supported on $[a_i, b_i]$ (\Cref{prop:hoeffding-general} in Appendix~\ref{app:tech-prelim}), we have
\begin{align*}
    \Pr\paren{\card{\Delta(x^{(t)})-\expect{\Delta(x^{(t)})}}\geq \eps}& \leq \exp\paren{-\frac{\eps^2 m^2}{\sum_{i=1}^{m}(b_i - a_i)^2}}\\
    &\leq \exp\paren{-\frac{\eps^2 m}{(b_i - a_i)}} \tag{$b_i-a_i\leq C^+-C^-$}\\
    &\leq \exp(-5\eps^2 \log{T}). \tag{$m\geq 5(C^+-C^-)\log{T}$.}
\end{align*}
Plugging $\eps=1$, we have 
\begin{align*}
    \Pr\paren{\card{\Delta(x^{(t)})-\expect{\Delta(x^{(t)})}}\geq 1}& \leq \frac{99}{100T}.
\end{align*}
Therefore, we could apply a union bound, and show that with probability at least $99/100$, it is always true that $\Delta(x^{(t)})>\expect{\Delta(x^{(t)})}-1$. By Assumption~\ref{assump:change-of-score}, we have $\expect{\Delta(x^{(t)})}\geq 1$, which means $\Delta(x^{(t)})>0$ for all steps with probability at least $99/100$.

The rest of the proof is very similar to the proof of \Cref{prop:muilti-step-single-opportunity}: by a similar calculation, we could obtain that 
\begin{align*}
    \delta-\delta'(T) & \geq \frac{\card{A\cap C_1}}{\card{A}}\cdot \mu^{C_1}_{A} - \frac{\card{B\cap C_1}}{\card{B}}\cdot \mu^{C_1}_{B} \\
    & + \frac{\card{A\cap C_3}}{\card{A}}\cdot \mu^{C_3}_{A} - \frac{\card{B\cap C_3}}{B}\cdot \mu^{C_3}_{B} \\\\
    & >0. \tag{applying Assumption~\ref{assump:real-advantage}}
\end{align*}
For the PoF, note that as long as $\alpha\geq \max\{\mu^{4}_{A}, \mu^{4}_{B}\}$, the ``even simpler'' investment policy is fair. 
As such, we could again let $T_0$ be the first step for all scores in $C_1$ and $C_3$ achieve $\xmax$ by the investment policy. A utility lower bound on the optimal fair policy would therefore be $((T-T_0)\cdot U^* - T_0\cdot (U^{+}-U^{-}))\cdot \card{C_1\cup C_3}$. In comparison, an upper bound for the optimal polity is $U^* \cdot T \cdot \card{C_1\cup C_3}$. 
    \begin{align*}
        \lim_{T\rightarrow \infty} \pof & \leq \lim_{T\rightarrow \infty}  1 - \frac{(T-T_0)\cdot U^* - T_0\cdot (U^{+}-U^{-})}{T\cdot U^*} \\
        &=0, \tag{$T_0$, $U^+$, $U^-$ are finite}
    \end{align*}
which is as claimed by the proposition.
\end{proof}

\FloatBarrier
\section{Omitted Details from Section~\ref{sec:exp}}
\label{sec:appendix-exp}

\subsection{Additional Implementation Details}
\label{sec:appendix-exp-additional-details}
In all experiments, we use $W_A=0.7$ and $W_B=0.3$, where $A$ is the group with the higher mean. For the synthetic data, the scores are drawn from two Gaussian distributions and then post-processed to integers ranging from 0 to 100. For the single-step setting, we use $80$ and $60$ as the means of groups $A$ and $B$, respectively. For the multi-step setting, we use $90$ and $70$ as the means of groups $A$ and $B$, respectively. In all experiments, we set the variance of both distributions to $30$. For the baseline instance for the synthetic dataset, we use $U^+=2$, $U^-=-2$, $C^+=2$, $C^-=-1$. For the baseline instance for the FICO dataset, we use $U^+=1$, $U^-=-2$, $C^+=7$, $C^-=-14$. For the high risk and cost instance for the synthetic dataset, we use $U^+=2$, $U^-=-20$, $C^+=2$, $C^-=-10$. All experiments are written in Python and were run on a MacBook.

\begin{figure}[ht!]
    \centering
    \includegraphics[width=0.8\textwidth]{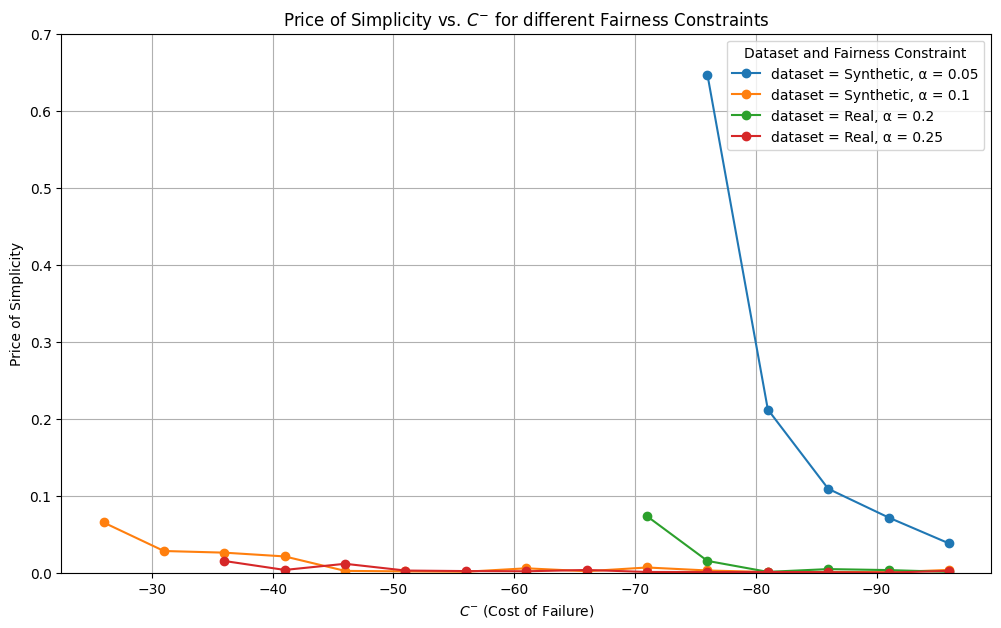}
    \caption{Price of Simplicity (PoS) for Different datasets and parameters.\label{fig:robust-assumpt-2}}
\end{figure}

\subsection{Sensitivity Analysis}
\label{sec:appendix-exp-sensitivity}
Motivated by~\Cref{thm:thresh_opt}, we also implemented an algorithm that searches over all possible thresholds for each group. This algorithm needs to potentially randomize at the threshold (see Definition~\ref{def:threshold-policy}). Beyond Assumption~\ref{assumpt:p}, the optimality of single-step threshold policies also requires Assumption~\ref{assumpt:struc_assumption}. We examine robustness to violations of this assumption by decreasing $C^-$ in the synthetic and FICO datasets while holding other parameters fixed. For a given $\alpha$, we measure the \emph{Price of Simplicity} (PoS), defined as $1$ minus the ratio of the utility achieved by the fair thresholding solution to that of the LP-based fair solution, which does not rely on Assumption~\ref{assumpt:struc_assumption}. PoS lies in $[0,1]$, with lower values indicating smaller utility loss from thresholding. Figure~\ref{fig:robust-assumpt-2} depicts the PoS for different combinations of $\alpha$ and dataset for the thresholding algorithm that \emph{does not} randomize at the threshold. We observe that PoS peaks under mild violations but rapidly approaches 0 as violations become severe. With randomizing over 10 values at the threshold, we observe that PoS remains close to 0 for all settings (not pictured), demonstrating the effectiveness of the simple thresholding policies.

\subsection{Multi-Step Setting}
\label{sec:appendix-exp-multi-step}

Figure~\ref{fig:multi-gap-appendix} is the replica of Figure~\ref{fig:multi-gap} over a time step of $T=50$ and using only a population of $N=10,000$. We observe that with a smaller population size, we still see the same behavior but the error bars are significantly larger.
\begin{figure}[ht!]
    \centering
    \includegraphics[width=0.8\textwidth]{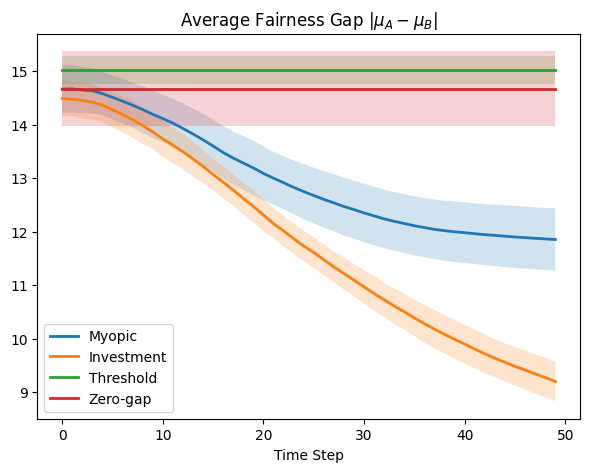}
    \caption{The average gap between group means over time.\label{fig:multi-gap-appendix}}
\end{figure}

\end{document}